%% file: arxiv.tex
\documentclass[11pt]{article} %
\usepackage{times}
\usepackage{url}            %
\usepackage{booktabs} 
\usepackage{wrapfig}%
\usepackage{amsfonts}       %
\usepackage{nicefrac}       %
\usepackage{microtype}      %
\usepackage{mkolar_definitions}
\input{math_commands.tex}

\usepackage{xspace}
\usepackage{multirow}
\usepackage{amsmath}
\usepackage{algorithm}
\usepackage{algorithmic}
\usepackage{color}
\usepackage{color, colortbl}
\usepackage{enumitem}
\usepackage{comment}
\usepackage{bm}
 \usepackage[left=2cm,top=2cm,right=2cm]{geometry}

\usepackage[scaled=.9]{helvet}

\usepackage{url}
\usepackage{authblk}
\usepackage{csquotes}

\usepackage[dvipsnames]{xcolor}
\usepackage[colorlinks=true, linkcolor=blue, citecolor=blue]{hyperref}

\usepackage{natbib}
\bibpunct{(}{)}{;}{a}{,}{,}

\definecolor{darkgreen}{rgb}{0,0.5,0}
\definecolor{darkred}{rgb}{0.7,0,0}
\definecolor{teal}{rgb}{0.3,0.8,0.8}
\definecolor{orange}{rgb}{1.0,0.5,0.0}
\definecolor{purple}{rgb}{0.8,0.0,0.8}

\definecolor{Gray}{gray}{0.9}

\input{notation.tex}

\renewcommand{\arraystretch}{2.}

\begin{document}
\title{Representation Learning for Online and Offline RL in Low-rank MDPs}
\author[1]{Masatoshi Uehara\thanks{mu223@cornell.edu}} 
\author[2]{Xuezhou Zhang\thanks{xz7392@princeton.edu }} 
\author[1]{Wen Sun \thanks{ws455@cornell.edu}}
\affil[1]{Department of Computer Science, Cornell University}
\affil[2]{Department of Electrical and Computer Engineering, Princeton University}
\date{}

\maketitle

\begin{abstract}

\input{abstract}

\end{abstract}

\section{Introduction}

\input{intro}
\input{related_work}

\input{prelim_arxiv_ver}

\input{online}

\input{offline}

\input{conclusion}

\bibliography{refs}

\appendix

\newpage

\section{Proof of the theoretical property of \ouralg{} }

\input{online_proof}

\section{Proof of the theoretical property of \ourofflinealg{}}

\input{offline_proof}

\input{auxiliary_lemma}

\input{tengyang_paper}

\end{document}

%% file: math_commands.tex
\usepackage{amsmath,amsfonts,bm}

\def\eqref#1{equation~\ref{#1}}

\def\1{\bm{1}}

\def\rd{{\textnormal{d}}}

\DeclareMathAlphabet{\mathsfit}{\encodingdefault}{\sfdefault}{m}{sl}
\SetMathAlphabet{\mathsfit}{bold}{\encodingdefault}{\sfdefault}{bx}{n}

\newcommand{\E}{\mathbb{E}}

\DeclareMathOperator*{\argmax}{arg\,max}

\DeclareMathOperator{\Tr}{Tr}

%% file: notation.tex
\newcommand{\ouralg}{\textsc{Rep-UCB}}
\newcommand{\ourofflinealg}{\textsc{Rep-LCB}}
\newcommand{\flambe}{\textsc{Flambe}}

%% file: abstract.tex
This work studies the question of \emph{Representation Learning in RL:} how can we learn a compact low-dimensional representation such that  on top of the representation we can perform RL procedures such as exploration and exploitation, in a sample efficient manner. We focus on the low-rank Markov Decision Processes (MDPs) where the transition dynamics correspond to a low-rank transition matrix. Unlike prior works that assume the representation is known (e.g., linear MDPs), here we need to learn the representation for the low-rank MDP. We study both the online RL and offline RL settings. For the online setting, operating with the same computational oracles used in \flambe \citep{Agarwal2020_flambe}----the state-of-art algorithm for learning representations in low-rank MDPs, we propose an algorithm \ouralg---{Upper Confidence Bound driven \textsc{Rep}resentation learning for RL}, which \emph{significantly improves} the sample complexity from $\widetilde{O}( A^9 d^7 / (\epsilon^{10} (1-\gamma)^{22}))$ for \flambe{} to $\widetilde{O}( d^4 A^2/ (\epsilon^2 (1-\gamma)^5)  )$ with $d$ being the rank of the transition matrix (or dimension of the ground truth representation), $A$ being the number of actions, and $\gamma$ being the discount factor. 
Notably, \ouralg{} is simpler than \flambe, as it directly balances the interplay between representation learning, exploration, and exploitation, while \flambe{} is an explore-then-commit style approach and has to perform reward-free exploration step-by-step forward in time. %
For the offline RL setting, we develop an algorithm that leverages pessimism to learn under a partial coverage condition: our algorithm is able to compete against \emph{any policy} as long as it is covered by the offline data distribution. \looseness=-1

%% file: intro.tex
When applying Reinforcement Learning (RL) to large-scale problems where data is complex and high-dimensional, learning effective transformations of the data, i.e., representation learning, can often significantly improve the sample and computation efficiency of the RL procedure. Indeed, several empirical works have shown that leveraging representation learning techniques developed in supervised or unsupervised learning settings can accelerate the search for good decision-making strategies \citep{silver2018general,stooke2021decoupling,srinivas2020curl,yang2021representation}.  However, representation learning in RL is far more subtle than it is for non-sequential and non-interactive learning tasks (e.g., supervised learning). Prior works have shown that even if one is given the magic representation that exactly linearizes the optimal policy \citep{du2019good} or the optimal value functions \citep{wang2020statistical,weisz2021exponential}, RL is still challenging (i.e., one may still need exponentially many samples to learn). This indicates that an effective representation that permits efficient RL needs to encode more information about the underlying Markov Decision Processes (MDPs). Despite the recent empirical success of representation learning in RL , its statistical guarantee and theoretical properties remain under-investigated. 

In this work, we study the representation learning question under the low-rank MDP assumption. Concretely, a low-rank MDP assumes that the MDP transition matrix admits a low-rank factorization, i.e., there exists two \emph{unknown} mappings $\mu(s'), \phi(s,a)$, such that $P(s' | s,a) = \mu(s')^{\top} \phi(s,a)$ for all $s,a,s'$, where $P(s' | s,a)$ is the probability of transiting to the next state $s'$ under the current state and action $(s,a)$. The representation $\phi$ in a low-rank MDP not only linearizes the optimal state-action value function of the MDP \citep{jin2020provably}, but also linearizes the transition operator.  A low-rankness assumption on large stochastic matrices is a common and natural assumption and has enabled successful development of algorithms for real world applications such as movie recommendation systems \citep{koren2009matrix}.  We note that a low-rank MDP strictly generalizes the linear MDP model \citep{YangLinF2019RLiF,jin2020provably} which assumes $\phi$ is known a priori. %
The unknown representation $\phi$ makes learning in low-rank MDPs much more challenging than that in linear MDPs since one can no longer directly use linear function approximations. On the other hand, the fact that  linear MDPs can be solved statistical and computational efficiently if $\phi$ is known a priori implies that if one could learn the representation of the low-rank MDP, one could then efficiently learn the optimal policy. \looseness=-1

 Indeed, prior works have shown that learning in low-rank MDPs is statistically feasible \citep{jiang2017contextual,sun2019model,du2021bilinear} via leveraging rich function approximators.  However, these algorithms are version space algorithms and are not computationally efficient. Recent work \flambe{} proposes an oracle-efficient algorithm\footnote{The oracle generally refers to supervised learning style empirical risk minimization oracle. We seek to design an algorithm that runs in polynomial time with each oracle call counting as $O(1)$. The reduction to supervised learning has lead to many successful provable and practical algorithms in contextual bandit \citep{agarwal2014taming,dudik2017oracle,foster2020beyond} and RL \citep{du2019provably,misra2020kinematic}.}
 that learns in low-rank MDPs with a polynomial sample complexity, where the computation oracle is Maximum Likelihood Estimation (MLE) operating under the standard supervised learning style Empirical Risk Minimization (ERM) setting. In this work, we follow the same setup from \flambe{} \citep{Agarwal2020_flambe}, and propose a new algorithm --- \emph{Upper Confidence Bound driven Representation Learning, Exploration and Exploitation (\ouralg)}, which can learn a near optimal policy for a low-rank MDP with a polynomial sample complexity and is oracle-efficient. Comparing to \flambe, our algorithm \emph{significantly improves} the sample complexity from $O(d^{7} A^9 / (\epsilon^{10}(1-\gamma)^{22})$ for \flambe{} to $O(d^4 A^2 / (\epsilon^2 (1-\gamma)^5)$, where $d$ is the rank of the transition matrix (or dimension of the true representation), $A$ is the number of actions, $\epsilon$ is the suboptimality gap and $\gamma\in[0,1)$ is the discount factor in the MDP. Our algorithm is also arguably much simpler than \flambe:  \flambe{} is an explore-then-commit algorithm, has to explore in a layer-by-layer forward way, and does not permit data sharing across different time steps. In contrast, \ouralg{} carefully trades  exploration versus exploitation 
by combining the reward signal and exploration bonus (constructed using the latest learned representation), and enables data sharing across all time steps.\footnote{Our algorithm and analysis can be easily extended to finite horizon non-stationary setting. We choose the discounted infinite horizon setting to contrast our results to \flambe: \flambe{} is \emph{not} capable of learning stationary policies under the discounted infinite horizon setting.} Our sample complexity nearly matches the ones from those computationally inefficient algorithms \citep{jiang2017contextual,sun2019model,du2021bilinear}. We summarize the comparison with the prior works that study representation learning in Table~\ref{tab:comparison}.

\renewcommand{\arraystretch}{1.7}
\begin{table*}[t] 
\centering
\vspace{-1cm}
\centering\resizebox{\columnwidth}{!}{
 \begin{tabular}{| c |  c | c | c |}
 \hline
 Methods & Setting & Sample Complexity & Computation \\  
  \hline\hline
 OLIVE \citep{jiang2017contextual} & Low Bellman rank & $\frac{d^2 A}{ \epsilon^2 (1-\gamma)^4 }$  & Inefficient  \\ 
 \hline
 Witness rank \citep{sun2019model}  & Low Witness rank & $\frac{d^2 A}{ \epsilon^2 (1-\gamma)^4 }$ & Inefficient \\
 \hline
 BLin-UCB \citep{du2021bilinear} & Bilinear Class & $\frac{d^2 A}{ \epsilon^2 (1-\gamma)^7 }$ & Inefficient \\
 \hline
  Moffle{} \citep{modi2021model} & Low-nonnegative-rank MDP & $\frac{d^6 A^{13} }{\epsilon^2\eta^5(1-\gamma)^5 }$  & Oracle efficient  \\
  \hline
   \flambe{} \cite{Agarwal2020_flambe} & Low-rank MDP & $\frac{d^7 A^9}{\epsilon^{10} (1-\gamma)^{22}}$ & Oracle efficient \\
 \hline
 \rowcolor{Gray}\ouralg{ }(Ours) & Low-rank MDP & $\frac{d^4 A^2}{\epsilon^2 (1-\gamma)^{5} }$ & Oracle efficient \\  
 \hline
 \end{tabular}}
 \vspace{-5pt}
 \label{tab:comparison}
 \caption{Comparison among different provable representation learning algorithms in online RL. Algorithms such as OLIVE, Witness rank, and BLin-UCB work for settings which are more general than low-rank MDPs and have tight sample complexity. However, these algorithms are version space algorithms and thus are not computationally efficient. Moffle is an oracle-efficient algorithm (with a much stronger oracle than the one in \flambe{} and ours), but the assumptions under which Moffle operates essentially imply that the MDP's transition has low non-negative matrix rank (nnr) (see the detailed discussion in the related work section). Note that a nnr is at least as large as and could be exponentially larger than the rank \citep{Agarwal2020_flambe}. Finally, \flambe{} operates under the same function approximation setting and the computation oracle as ours. Our algorithm \emph{significantly improves} the sample complexity from \flambe{} in \emph{all parameters}.  Note the horizon dependence is not exactly comparable as these prior works originally considered the finite horizon setting with nonstationary transition, and we convert their results to the discounted setting by simply replacing the finite horizon H by $\Theta(1/(1-\gamma))$.
 }
\end{table*}

In addition to the online exploration setting, we also show that our new techniques can be directly used for designing offline RL algorithms for low-rank MDPs under partial coverage. More specifically, we propose an algorithm \ourofflinealg{}---\emph{{L}ower {C}onfidence {B}ound driven {Rep}representation Learning for offline RL}, that given an offline dataset, can learn to compete against any policy (including history-dependent policies) as long as it is covered by the offline data where the coverage is measured using the relative condition number \citep{agarwal2021theory} associated with the ground truth representation. Thus, our offline RL result generalizes prior offline RL works on linear MDPs \citep{JinYing2020IPPE,zhang2021corruption} which assume representation is known a priori and use linear function approximation. Computation-wise, our approach uses one call to the MLE computation oracle, and hence is oracle-efficient. \ourofflinealg{} is \emph{the first oracle efficient offline algorithm for low-rank MDP enjoying the aforementioned statistical guarantee.} See Section~\ref{sec:related_work} for a more detailed comparison with the existing literature on representation learning in offline RL.  \looseness=-1

\paragraph{Our contributions.} We develop new representation learning RL algorithms that enable sample efficient learning in low-rank MDPs under both online and offline settings:
\begin{enumerate}[leftmargin=0.9cm]

\item In the online episodic learning setting, our new algorithm \ouralg{} integrates representation learning, exploration, and exploitation together, and significantly improves the sample complexity of the prior state-of-art algorithm \flambe;

\item  In the offline learning setting, we propose a natural concentrability coefficient (i.e., relative condition number under the true representation) that captures the partial coverage condition in low-rank MDP, and our algorithm \ourofflinealg{} learns to compete against any policy (including history-dependent ones) under such a partial coverage condition. 
\end{enumerate}

%% file: related_work.tex
\section{Related Work}
\label{sec:related_work}

\paragraph{Online Setting} We list the comparison as follows, which is summarized in Table \ref{tab:comparison}. 

\flambe{} \citep{Agarwal2020_flambe} was a state-of-the-art oracle-efficient algorithm for low-rank MDPs.  In all parameters, the statistical complexity is much worse than  \ouralg{ }. 
Our algorithm and \flambe{} operate under the same computation oracle. \flambe{} does not balance exploration and exploitation, and uses explore-then-committee style techniques (i.e., constructions of absorbing MDPs \citep{brafman2002r}) which results in its  worse sample complexity. 
    
With a more complex oracle, Moffle \citep{modi2021model} is a model-free algorithm for low-rank MDPs, with  two additional assumptions: (1) the transition has low \emph{non-negative rank (nnr)}, and (2) reachability in latent states.  The first assumption significantly restricts the scope of low-rank MDPs as there are matrices whose nnr is exponentially larger than the rank \citep{Agarwal2020_flambe}. The sample complexity of Moffle can scale $O(d^6|\Acal|^{13}/(\epsilon^2\eta^5 (1-\gamma)^5))$, where $\eta$ is the reachability probability, and $1/\eta$ could be as large as  $nnr^{1/2}$ (Proposition 4 in \citet{Agarwal2020_flambe}), which essentially means that Moffle has a polynomial dependence on the nnr. Thus, Moffle needs the nnr of the transition matrix to be small.

OLIVE \citep{jiang2017contextual}, Witness rank \citep{sun2019model} and Bilinear-UCB \citep{du2021bilinear}, when specialized to low-rank MDPs, have slightly tighter dependence on $d$ (e.g., $O(d^2 / \epsilon^2)$). But these algorithms are computationally inefficient as they are version space algorithms.  \cite{dann2021agnostic} shows that with a policy class, solving a low-rank MDP can take $\Omega(2^{d})$ samples. In this work, similar to Witness rank \citep{sun2019model} and \flambe, we use function approximators to model the transition. Thus our positive result is not in contradiction to the result from \cite{dann2021agnostic}. \looseness=-1 
    
VALOR \citep{dann2018oracle}, PCID \citep{du2019provably}, HOMER \citep{misra2020kinematic}, RegRL \citep{foster2020instance}, and the approach from \cite{feng2020provably} are algorithms for block MDPs which is a more restricted setting than low-rank MDPs. These works require additional assumptions such as deterministic transitions \citep{dann2018oracle}, reachability \citep{misra2020kinematic,du2019provably},  strong Bellman closure \citep{foster2020instance}, and strong unsupervised learning oracles \citep{feng2020provably}. \looseness=-1

\paragraph{Offline Setting} We discuss related works in offline RL. %
 
\citet{uehara2021pessimistic} obtained similar statistical results for offline RL on low-rank MDPs. Though the sample complexity in their algorithm is slightly tighter, our algorithm is oracle-efficient, while the CPPO algorithm from \citet{uehara2021pessimistic} is a version space algorithm. 

\citet{XieTengyang2021BPfO} propose a (general) pessimistic model-free algorithm in the offline setting. We can also apply their algorithm to low-rank MDPs and show some finite-sample guarantee. However, it is unclear whether the final bounds in their results can be characterized by the relative condition number only using the true representation, and whether they can compete with history-dependent policies. 
Thus, our result is still considered superior on low-rank MDPs.
The detail is given in \pref{sec:comparison}. \looseness=-1

In addition to the above two works, the pessimistic approach in offline RL has been extensively investigated. Empirically, it can work on simulation control tasks \citep{Kidambi2020,Yu2020,kumar2020conservative,Liu2020,ChangJonathanD2021MCSi}. On the theoretical side,  pessimism allows us to obtain the PAC guarantee on various models when a comparator policy is covered by offline data in some forms \citep{JinYing2020IPPE,RashidinejadParia2021BORL,YinMing2021NORL,zanette2021provable,zhang2021corruption,ChangJonathanD2021MCSi}. However, these algorithms and their analysis rely on a \emph{known representation} and linear function approximation. %

%% file: prelim_arxiv_ver.tex
\section{Preliminaries}
\label{sec:prelim}
We consider an episodic discounted infinite horizon Markov Decision Process $\Mcal=\langle \Scal,\Acal,P,r,\gamma,d_0\rangle $ specified by a state space $\Scal$, a discrete action space $\Acal$, a transition model $P:\Scal\times \Acal\to \Delta(\Scal)$, a reward function $r:\Scal \times \Acal \to \mathbb{R}$, a discount factor $\gamma\in [0,1)$, and an initial distribution $d_0\in \Delta(\Scal)$. To simplify the presentation, we assume $r(s,a)$ and $d_0$ are known (e.g., when $d_0$ is a probability mass only on $s_0$, agent always starts from a fixed initial state $s_0$)\footnote{Extension to the unknown case is straightforward. Recall the major challenging of RL is due to the unknown transition model.}. Following prior work \citep{jiang2017contextual,sun2019model}, we assume \emph{trajectory reward is normalized}, i.e., for any trajectory $\{s_h,a_h\}_{h=0}^{\infty}$, we have $\sum_{h=0}^{\infty} \gamma^h r(s_h,a_h) \in [0,1]$. Since the ground truth $P^{\star}$ is unknown, we need to learn it by interacting with environments in an online manner or utilizing offline data at hand.  We remark that the extension of our all results  to the finite horizon nonstationary case is straightforward. %

\begin{figure}
\begin{center}
        \includegraphics[width=0.5\textwidth]{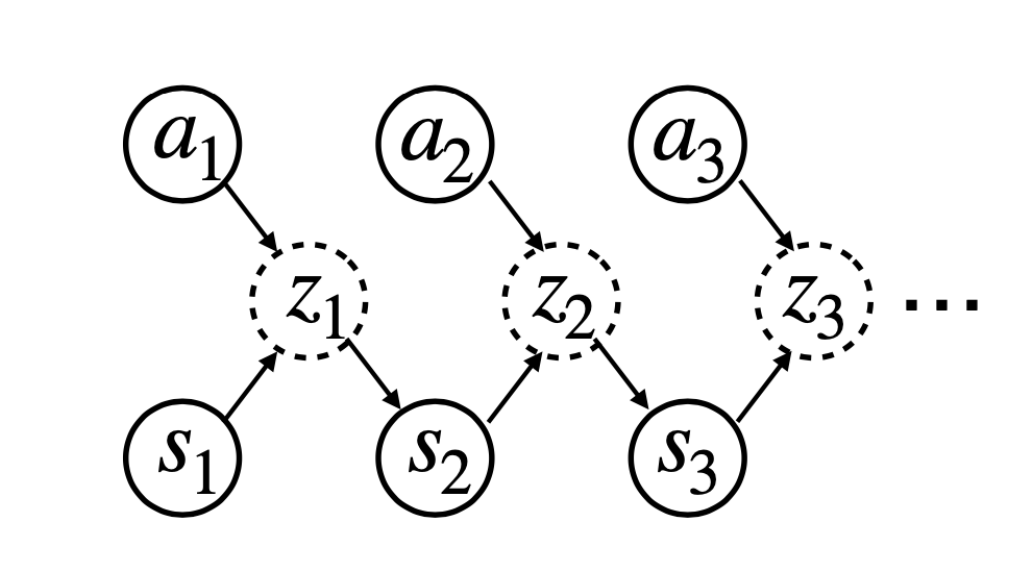}
\end{center}
  \caption{A latent state model captured by low-rank MDP. Here $\phi^\star(s,a)$ is a distribution over a discrete latent state space $\mathcal{Z}$. Note that this is still a Markovian model as there is no transition among latent states.}
  \label{fig:latent_rep}
\end{figure}

We use the following notation. Given a policy $\pi:\Scal \to \Delta(\Acal)$, which maps from state to distribution over actions and some model $P$, we define the value function $V^{\pi}_{P}(s) = \mathbb{E}\left[ \sum_{h=0}^{\infty} \gamma^h r(s_h,a_h) | s_0 = s, P, \pi \right]$ to represent the expected total discounted reward of $\pi$ under $P$ starting at $s$. Similarly, we define the state-action Q function $Q^{\pi}_P(s,a) := r(s,a) + \gamma \EE_{s'\sim P(\cdot | s,a)} V^{\pi}_P(s')$. The expected total discounted reward of a policy $\pi$ under transition $P$ and reward $r$ is denoted as $V^{\pi}_{P,r} := \EE_{s_0\sim d_0} V^{\pi}_P(s_0)$.
We define the discounted state-action occupancy distribution $d^{\pi}_P(s,a)=(1-\gamma)\sum_{t=0}^{\infty} \gamma^t d^{\pi}_{P,t}(s,a)$, where $d^{\pi}_{P,t}(s,a)$ is the probability of $\pi$ visiting $(s,a)$ at time step $t$ under $\pi$ and $P$. We slightly abuse the notation, and denote $d^{\pi}_{P}(s)$ as the state visitation, which is  equal to $\sum_{a\in\Acal} d^{\pi}_P(s,a)$. When $P$ is the ground truth transition model $P^\star$, we drop the subscript and simply use $d^{\pi}$ to denote its state-action distribution. Unless otherwise noted, $\Pi$ denotes the class of all polices $\{S\to \Delta(\Acal)\}$. We denote total variation distance of $P_1$ and $P_2$ by $\|P_1-P_2\|_1$. Finally, given a vector $a$, we define $\|a\|_2=\sqrt{a^{\top}a},\|a\|_{B}=\sqrt{a^{\top}Ba}$. $c_0,c_1,\cdots$ are universal constants.

We study low-rank MDPs defined as follows \citep{jiang2017contextual,Agarwal2020_flambe}. The conditions on the upper bounds of the norm of $\phi^{\star},\mu^{\star}$ are just for normalization.   
\begin{definition}[Low-rank MDP]
A transition model $P^{\star}:\Scal\times \Acal \to \Delta(\Acal)$ admits a low rank decomposition with rank $d\in \NN$
if there exists two embedding functions $\phi^{\star}\, \mu^{\star}$ such that 
\begin{align*}
    \forall s,s'\in \Scal, a\in \Acal: P^{\star}(s'\mid s,a)=\mu^{\star}(s')^{\top}{\phi^{\star}(s,a)}
\end{align*}
where $\|\phi^{*}(s,a)\|_2\leq 1$ for all $(s,a)$ and for any function $g:\Scal \to [0,1]$, $\|\int \mu^{\star}(s)g(s)\rd(s)\|_2\leq \sqrt{d}$. An MDP is a low rank MDP if $P^{\star}$ admits such a low rank decomposition.  
\end{definition}

Low-rank MDPs capture the latent variable model \citep{Agarwal2020_flambe} shown in \pref{fig:latent_rep} where $\phi^\star(s,a)$ is a distribution over a discrete latent state space $\mathcal{Z}$. The block-MDP model \citep{du2019provably} is a special  instance of the latent variable model with $\phi^\star(s,a)$ being a one-hot encoding vector. 
Note the linear MDPs \citep{YangLinF2019RLiF,jin2020provably} assume $\phi^\star$ is known.  %

Next, we explain two settings: the online learning setting and the offline learning setting. Then, we present our function approximation setup and computational oracles.

\paragraph{Episodic Online learning}
In online learning, our overall goal is  to learn a stationary policy $\hat \pi$ so that it maximizes $V^{\hat \pi}_{P^{\star},r}$, where $P^\star$ is the ground truth transition.  We assume that we operate under the episodic learning setting where we can only reset to states sampled from the initial distribution $d_0$ (e.g., to emphasize the  challenge from exploration, we can consider the special case where we can only reset to a fixed $s_0$). In the episodic setting, given a policy $\pi$, sampling a state $s$ from the state visitation $d_P^{\pi}$ is done by the following \emph{roll-in} procedure: starting at $s_0\sim d_{0}$, at every time step $t$, we terminate and return $s_t$ with probability $1-\gamma$, and otherwise we execute $a_t\sim \pi(s_t)$ and move to $t+1$, i.e., $s_{t+1} \sim P(\cdot | s_t,a_t)$. Such a sampling procedure is widely used in the policy gradient and policy optimization literature (e.g., \citep{kakade2002approximately,agarwal2021theory,agarwal2020pc}). 
\looseness=-1 %

\paragraph{Offline learning} In the offline RL, we are given a static dataset in the form of quadruples:
\begin{align*}\textstyle
    \Dcal=\{s^{(i)}, a^{(i)}, r^{(i)}, s'^{(i)}\}_{i=1}^n\sim \rho(s,a)\delta(r=r(s,a))P^{\star}(s'\mid s,a). 
\end{align*} 
For simplicity, we assume $\rho=d^{\pi_b}_{P^{\star}}$, where $\pi_b \in \Scal\to \Delta(\Acal)$ is a fixed behavior policy. We  denote $\EE_{\Dcal}[f(s,a,s')]=1/n\sum_{(s,a,s')\in \Dcal}f(s,a,s')$. To succeed in offline RL, we in general need some coverage property of $\rho$. One common assumption is that $\rho$ globally covers every possible policies' state-action distribution, i.e., $\max_{\pi,s,a} \frac{d^{\pi}_{P^\star}(s,a)}{\rho(s,a)} < \infty$ \citep{antos2008learning}. In this work, we relax such a global coverage assumption and work under the \emph{partial} coverage condition where $\rho$ may not cover distributions of all possible policies. Instead of competing against the optimal policy under the global coverage, we aim to compete against any policies covered by the offline data. %
In section~\ref{sec:offline}, we will precisely define the partial coverage condition using the concept of the relative condition number \citep{agarwal2021theory}. %
\looseness=-1

\paragraph{Function approximation setup and computational oracles} 

Since $\mu^\star$ and $\phi^\star$ are unknown, we use function classes to capture them. Our function approximation and computational oracles are \emph{exactly the same as the ones used in \flambe}. For completeness, we state the function approximation and computational oracles below.  %

\begin{assum}\label{assum:basic}
We have a model class $\Mcal = \{ (\mu, \phi): \mu\in \Phi, \phi\in\Phi \}$, where $\mu^\star\in \Phi$, $\phi^\star\in\Phi$.%
\end{assum} 
Following the norm bounds on $\mu^\star, \phi^\star$  we similarly assume that the same norm bounds hold for our function approximator, i.e., for any $\mu\in \Phi,\phi\in \Phi$, $\|\phi(s,a)\|_{2} \leq 1$, 
$\forall (s,a)$ and $\|\int \mu(s)g(s)\rd(s)\|_2\leq \sqrt{d},\forall g:\Scal \to [0,1]$, and $\int \mu^{\top}(s')\phi(s,a)\rd(s')=1,\,\forall (s,a)$. \looseness=-1

As  for computational oracles, we use  a supervised learning style MLE oracle. 
\begin{definition}[Maximum Likelihood Oracle (MLE)] \label{asm:mle}
Consider the model class $\Mcal$ and a dataset $\mathcal{D}$ in the form of $(s,a,s')$, the MLE oracle returns the maixmum likelihood estimator $\hat P := (\hat\mu,\hat\phi) = \argmax_{(\mu,\phi)\in\Mcal} \EE_{\Dcal} \ln (\mu(s')^{\top}\phi(s,a))$.
\end{definition}
We also invoke a planning procedure for \emph{known linear MDPs with potentially nonlinear rewards}, which can be done in polynomial time (after all, we know that online learning in linear MDPs can be done statistically and computationally efficient). 
Given a reward $r$ and a model $P:= (\mu,\phi)$ with $P(s' | s,a) = \mu(s')^{\top} \phi(s,a)$ (i.e., a \emph{known linear transition} with a \emph{known} feature $\phi$), we can compute the optimal policy $\arg\max_{\pi} V^{\pi}_{P,r}$ by standard least square value iteration which uses linear regression. 
A planning procedure for a known linear MDP is also used in \flambe, see Section 5.1 in \citet{Agarwal2020_flambe} for details of how to implement this procedure for a known linear MDP with polynomial computation complexity.   \looseness=-1

%% file: online.tex
\section{Representation Learning in Online Setting}
\label{sec:online}

We consider the online episodic learning setting where the agent can only reset based on the initial state distribution $d_0$. To find a near-optimal policy for a low-rank MDP efficiently, we need to carefully interleave representation learning, exploration, and exploitation.

\subsection{Algorithm}

We present our proposed algorithm in the online setting which is described in \pref{alg:online_algorithm}.  We first describe the data collection process. Every iteration, \pref{alg:online_algorithm} rollouts its current policy $\pi$ to collect a tuple $(s,a,s',a',\tilde s)$ where $s\sim d_{P^\star}^{\pi}$, $a\sim U(\Acal), s'\sim P^\star(\cdot | s,a), a' \sim U(\Acal), \tilde s \sim P^{\star}(\cdot \mid s,a)$ where $U(\Acal)$ is a uniform distribution over actions (note that we take two uniform actions here). Recall that to sample $s\sim d^{\pi}_{P^\star}$, we start at $s_0\sim d_{0}$, at every time step $t$, we terminate and return $s_t$ with probability $1-\gamma$, and otherwise we execute $a_t\sim \pi(s_t)$ and move to $t+1$, i.e., $s_{t+1} \sim P^\star(\cdot | s_t,a_t)$. Thus collecting one tuple requires exactly one roll-in (i.e., one trajectory, and it is easy to verify that with high probability the roll-in terminates with $\tilde{O}(1/(1-\gamma))$ steps which is often called the effective horizon). 

After collecting new data and concatenating it with the existing data, we perform representation learning,  i,e, learning a factorization and a representation by MLE (line~\ref{line:mle}), set the bonus based on the learned feature (Eq.~\ref{eq:bonus_online}), and update the policy via planning inside the learned model with the bonus-enhanced reward (Line~\ref{line:plan_online}).  Note the learned transition $\hat{P}$ from MLE is linear with respect to the learned feature $\hat\phi$, and planning in a known linear MDP is known as computationally efficient \citep{jin2020provably} (see the explanation after Definition~\pref{asm:mle} as well). \looseness=-1

\begin{algorithm}[t] 
\caption{UCB-driven representation learning, exploration, and exploitation (\ouralg) } \label{alg:online_algorithm}
\begin{algorithmic}[1]
  \STATE {\bf  Input:} Regularizer $\lambda_n$, parameter $\alpha_n$, Models $\Mcal=\{( \mu,\phi):\mu\in \Psi,\phi\in \Phi\}$, Iteration $N$
  \STATE Initialize $\pi_0(\cdot\mid s)$ to be uniform; set $\Dcal_0 = \emptyset$, $\Dcal'_0 = \emptyset$
  \FOR{episode $n=1,\cdots,N$ } 
  \STATE Collect  a tuple $(s,a,s',a',\tilde s)$ with \begin{equation*}s\sim d_{P^\star}^{\pi_{n-1}}, a\sim U(\Acal), s'\sim P^\star(\cdot | s,a),a' \sim U(\Acal),\tilde s\sim P^\star(\cdot | s',a') \vspace{-5pt}\end{equation*}
  \STATE Update datasets by adding triples $(s,a,s')$ and $(s',a',\tilde s)$: $$\Dcal_n = \Dcal_{n-1} + \{(s,a,s')\}, \quad \Dcal'_n = \Dcal'_{n-1} + \{(s',a', \tilde s)\}$$ %
  
  \STATE Learn representation via ERM (i.e., MLE):  $$
      \hat P_n:= ( \hat \mu_n, \hat \phi_n) =\argmax_{( \mu,\phi )\in \Mcal} \E_{ \Dcal_n + \Dcal'_n }\left[\ln \mu^{\top}(s') \phi(s,a)\right]$$ \label{line:mle} %
  \STATE Update empirical covariance matrix $\hat\Sigma_n = \sum_{s,a\in\Dcal_n} \hat\phi_n(s,a) \hat\phi_n(s,a)^{\top} + \lambda_n I$
  \STATE Set the exploration bonus: %
     \begin{equation}
     \label{eq:bonus_online}
      \hat b_n(s,a) :=\min \prns{\alpha_n \sqrt{\hat \phi_n(s,a)^{\top}\hat \Sigma^{-1}_{n}\hat \phi_n(s,a)},2} %
     \end{equation}
  \STATE  Update policy $\pi_n=\argmax_{\pi}V^{\pi}_{\hat P_n,r+\hat b_n}$  \label{line:plan_online}
  \ENDFOR 
  \STATE \textbf{Return } $\pi_1,\cdots,\pi_N$
\end{algorithmic}
\end{algorithm}

\subsection{Analysis}

We show that our algorithm has the following PAC bound.  %
\begin{theorem}[PAC Bound for \ouralg]\label{thm:online} Fix $\delta\in(0,1), \epsilon\in(0,1)$.
Let $\hat \pi$ be a uniform mixture of $\pi_1,\cdots,\pi_N$ and $\pi^{\star}:=\argmax_{\pi} V^{\pi}_{P^{\star},r}$ as the  optimal policy. By setting parameters as follows: 
\begin{align*}
 \alpha_n = O\prns{\sqrt{\prns{|\Acal|+d^2}\gamma\ln(|\Mcal|n/\delta)  }},\quad \lambda_n=O\prns{d\ln (|\Mcal|n/\delta)}, %
\end{align*} with probability at least $1-\delta$, we have $$V^{\pi^\star}_{P^\star,r} - V^{\hat\pi}_{P^\star, r} \leq \epsilon,$$ 
where the number of collected samples is at most
\begin{align*}
    O\prns{  \frac{d^4|\Acal|^2 \ln(|\Mcal|/\delta) }{(1-\gamma)^5 \epsilon^2}\cdot  \nu  }, 
 \end{align*} where $\nu$ only contains log terms and the dependence on $|\Mcal|$ is at most $\ln ( \ln(|\Mcal|))$, i.e., $$\nu := O\left( \ln\left( \frac{d^4|\Acal|^2 \ln(|\Mcal|/\delta) }{(1-\gamma)^5\delta \epsilon^2}  \ln^{2}\prns{1+  \frac{d^4|\Acal|^2  \ln(|\Mcal|/\delta) }{(1-\gamma)^5 \epsilon^2}}  \right) \cdot \ln^{2}\prns{1+  \frac{d^4|\Acal|^2  \ln(|\Mcal|/\delta) }{(1-\gamma)^5 \epsilon^2}} \right).$$
\end{theorem}
The theorem shows that \ouralg{} learns in low-rank MDPs in a statistically efficient and oracle-efficient manner. To the best of our knowledge, this algorithm has the best sample complexity among all oracle efficient algorithms for low-rank MDPs. %

\paragraph{Highlight of the analysis} %
Below we highlight our key lemmas and proof techniques. %

\emph{First}, why is learning in a low-rank MDP harder than learning in models with linear structures? Unlike standard linear models such as linear MDPs \citep{YangLinF2019RLiF,jin2020provably}, KNRs \citep{Kakade2020,abbasi2011regret,mania2020active,song2021pc}, and GP / kernel models \citep{chowdhury2019online,curi2020efficient}, we cannot get uncertainty quantification on the model in a point-wise manner. When models are linear, one  can get the following style of point-wise uncertainty quantification for the learned model $\hat{P}$: $\forall s,a: {\ell}( \hat{P}(\cdot|s,a), P^\star(\cdot | s,a)) \leq \sigma(s,a)$ where $\sigma(s,a)$ is the uncertainty measure, and ${\ell}$ is some distance metric (e.g., $\ell_1$ norm). With proper scaling, the uncertainty measure $\sigma(s,a)$ is then used for the bonus. For example, in linear MDPs (i.e., low-rank MDP with known feature $\phi^\star$), given a dataset $\Dcal = \{s,a,s'\}$, we can learn a non-parametric model $\hat{P}(s'|s,a) := \hat{\mu}(s')^{\top}\phi^\star(s,a)$ , and get point-wise uncertainty quantification: %
\begin{align}  \label{eq:cs}
    \forall (s,a),\lvert  \int f(s')\hat\mu^{\top}(s') \phi^\star(s,a)\rd(s') -\int f(s'){\mu^\star}^{\top}(s') \phi^\star(s,a)\rd(s') \rvert \leq c \|\phi^\star(s,a)\|_{\Sigma^{-1}_{\phi^\star}}
\end{align}
for some family of functions $f:\Scal\to \mathbb{R}$ with $\Sigma_{\phi^\star} = \sum_{s,a\in\Dcal} \phi^\star(s,a)\phi^\star(s,a) + \lambda I$ \citep{lykouris2021corruption,neu2020unifying}. To set the scaling $c$ properly, since $\phi^\star$ is known a priori, the linear regression analysis applies here, and one can apply Cauchy-Schwarz inequality to the LHS of \pref{eq:cs}
to pull out $\phi^\star$ and get an upper bound in the form of 
\begin{align*}
  \underbrace{\|\phi^\star(s,a)\|_{\Sigma^{-1}_{\phi^\star}}}_{(a)}  \underbrace{ \|\int f(s)\{\hat \mu(s)-\mu^{\star}(s)\}\rd(s) \|_{\Sigma_{\phi^\star}} }_{(b)}
\end{align*}
where $c$ is set to be the linear regression training error measured in the term $(b)$ above. 

However, when we jointly learn $\mu$ and $\phi$, since nonlinear function approximation is used\footnote{There are nonlinear models using Eluder dimension \citep{russo2014learning} as the complexity measure. However, to the best of our knowledge, the only known examples that admit low Eluder dimension are linear models and generalized linear models with strong assumptions on the link function.\looseness=-1}, we cannot get point-wise uncertainty quantification via linear regression-based analysis. We stress that our bonus is not designed to capture the uncertainty quantification on the model error between $\hat{P}(\cdot|s,a) = \hat\mu^{\top} \hat\phi(s,a)$ and $P^\star(\cdot|s,a) = {\mu^\star}^{\top}\phi^\star(s,a)$ in a point-wise way, which is not tractable as $\hat{P}$ and $P^\star$ does not even share the same representation. Instead,  the bonus is carefully designed so that it only provides \emph{near-optimism at the initial state distribution}. This is formalized as follows. 
\begin{lemma}[Almost Optimism at the Initial State Distribution]\label{lem:optimism_main}
Set the parameters as in \pref{thm:online}. With probability $1-\delta$, 
\begin{align*}
 \forall n\in [1,\cdots,N],\forall \pi\in \Pi, V^{\pi}_{\hat P_n,r+\hat b_n}-V^{\pi}_{P^{\star},r}\geq -c_1\sqrt{\frac{|\Acal|\ln(|\Mcal|n/\delta)(1-\gamma)^{-1}}{n}}. 
\end{align*}
\end{lemma}
We remark that the idea of  optimism with respect to the initial state distribution has been used in prior works \citep{jiang2017contextual,sun2019model,du2021bilinear,zanette2020learning}. However, these algorithms are not computationally efficient (i.e., they use version space instead of reward bonus), and their version-space based analysis is different from ours.

\paragraph{Proof sketch for Lemma~\ref{lem:optimism_main}} We start by using the simulation lemma (\pref{lem:performance}) \emph{inside the learned model} which is important since our bonus $\hat{b}_n$ uses $\hat\phi_n$ associated with the learned model $\hat{P}_n$: %
\begin{align*}
V^{\pi}_{\hat{P}_n, r+ \hat{b}_n} - V^{\pi}_{P^\star,r} & =  (1-\gamma)^{-1}\EE_{s,a\sim d^{\pi}_{\hat{P}_n}}\left[ \hat{b}_n(s,a) + \gamma \EE_{s'\sim \hat{P}_n(\cdot | s,a)} V^{\pi}_{P^\star}(s') - \gamma \EE_{s'\sim {P^\star}(\cdot | s,a)} V^{\pi}_{P^\star}(s')    \right]\\
& \geq (1-\gamma)^{-1}\EE_{s,a\sim d^{\pi}_{\hat{P}_n}}\left[ \hat{b}_n(s,a) - \| \hat{P}_n(\cdot | s,a) - {P^\star}(\cdot | s,a)  \|_1   \right],
\end{align*} from where we show that $\EE_{s,a\sim d^{\pi}_{\hat{P}_n}}\| \hat{P}(\cdot|s,a) - P^\star(\cdot | s,a)  \|_1 $ \emph{as a whole} nearly lower bounds the average bonus $\EE_{s,a\sim d^{\pi}_{\hat{P}_n}}[ \hat{b}_n(s,a)]$.  Thus the proof of optimism is fundamentally different from the proofs in tabular and linear MDPs which are done via induction in a point-wise manner. The detailed procedure is illustrated in Lemma~\ref{lem:optimism} in Appendix~\ref{ape:online}.

\emph{Second}, our bonus is using representation $\hat{\phi}_n$ that is being updated every episode, and our empirical covariance matrix $\hat\Sigma_n$ is also updated whenever we update $\hat\phi_n$, which means that standard elliptical potential based analysis (i.e., analysis used in linear bandits/MDPs with known features) cannot work here as our feature changes every episode. Instead, in our analysis, we have to keep tracking a potential function that is defined using the unknown ground truth representation $\phi^\star$, i.e., the elliptical potential $\| \phi^\star(s,a) \|^2_{\Sigma^{-1}_{\rho_n,\phi^\star} }$, where $$\Sigma_{\rho_n,\phi^\star} = n \EE_{(s,a)\sim \rho_n} \phi^\star(s,a) \phi^\star(s,a)^{\top} + \lambda_n I,$$ and $\rho_n(s,a) = \sum_{i=0}^{n-1} d^{\pi_i}_{P^\star}(s,a) / n$. Since this potential function uses the fixed representation $\phi^\star$, we can apply the standard elliptical potential argument to track the progress that our algorithm makes during learning. Below we illustrate the procedure of linking the bonus under $\hat{\phi}_n$ to the potential function $\| \phi^\star(s,a) \|^2_{\Sigma^{-1}_{\rho_n,\phi^\star} }$ defined with respect to the true feature $\phi^\star$. Note that this potential function is only used in analysis.

\paragraph{Linking bonus under $\hat\phi_n$ to the elliptical potential function under $\phi^\star$}
With near optimism, using the simulation lemma (\pref{lem:performance}) inside the real model, we can upper bound the per-iteration regret as follows: 
\begin{align*}
    V^{\pi^{\star}}_{P^{\star},r}-  V^{\pi_n}_{P^{\star},r}\leq (1-\gamma)^{-1}\E_{(s,a) \sim d^{\pi_n}_{P^{\star}}}[\hat b_n(s,a)+(1-\gamma)^{-1} f_n(s,a)] + \sqrt{|\Acal| \zeta_n(1-\gamma)^{-1} }, 
\end{align*} 
where $\zeta_ n = \tilde{O}( 1/n)$, and $f_n(s,a) :=\|\hat P_n(\cdot\mid s,a)-P^{\star}(\cdot \mid s,a)\|_1$. To connect the first term in the right-hand side of the above inequality to the elliptical potential  under the fixed feature $\phi^\star$, we show that for any function $g\in \Scal\times\Acal\to [0,B]$ for $B\in\mathbb{R}^+$,  %
\begin{align*}
    \E_{(s,a)\sim d^{\pi_n}_{P^{\star}}}[g(s,a)] & \leq  (1-\gamma)^{-1}\E_{(s,a)\sim d^{\pi_n}_{P^{\star}}}\left[\| \phi^{\star}(s,a)\|_{\Sigma^{-1}_{\rho_n,\phi^{\star}}}\right]\sqrt{n\gamma |\Acal| \E_{\rho'_n}[g^2(s,a)] +\gamma \lambda_n d B^2 } \\
    & \qquad \qquad  + \sqrt{(1-\gamma)|\Acal| \EE_{\rho'_n} [g^2(s,a)] }, 
\end{align*} 
where $\rho'_n(s,a)=1/n \sum_{i=0}^{n-1} d^{\pi_i}(s)u(a)$ and $u(a)=1/|\Acal|$. See \pref{lem:useful} in Appendix~\pref{ape:online}.
By substituting $g$ with $\hat b_n +  f_n/(1-\gamma)$, the first term of the RHS of the above inequality can be upper bounded as:
\begin{align*}
2(1-\gamma)^{-1}\underbrace{\E_{(s,a)\sim d^{\pi_n}_{P^{\star}}}\left[\| \phi^{\star}(s,a)\|_{\Sigma^{-1}_{\rho_n,\phi^{\star}}}\right]}_{(\Gcal_1)}\underbrace{\sqrt{ n|\Acal|\E_{\rho'_n}\bracks{\frac{f^2_n(s,a)}{(1-\gamma)^2}+\hat b^2_n(s,a)}+ \lambda_n d }}_{(\Gcal_2)}. 
 \end{align*}
 In the term ($\Gcal_2$), we expect $ n \EE_{\rho'_n}[f_n^2(s,a)]$ to be $O(1)$ as $\EE_{ \rho'_n}[f^2_n(s,a)]$ is in order of $1/n$ due to the fact that it is the generalization bound of the MLE estimator $\hat{P}_n$ which is trained on the data drawn from $\rho'_n$ . For $n \EE_{\rho'_n}[\hat b^2_n(s,a)]$, we expect it to be  in the order of $d$ as the (unnormalized) data covariance matrix $\hat{\Sigma}_n$ in the bonus $\hat{b}_n$ uses training data from $\rho'_n$, i.e., we are measuring the expected bonus under the training distribution. 
 In other words, the term ($\Gcal_2$) scales in order of $\text{poly}(d)$. For the term $(\Gcal_1)$, since it contains the potential function based on $\phi^\star$, the sum of the term ($\Gcal_1$) over all episodes can be controlled by the standard elliptical potential argument (see \pref{lem:reduction} and \pref{lem:potential}). This concludes the proof sketch of our main theorem.

\textbf{In summary}, our analysis relies on the standard idea of optimism in the face of uncertainty, but with novel techniques to achieve optimism under nonlinear function approximation with the MLE supervised learning style generation bound, and to track regret under changing representations.

%% file: offline.tex
\section{Representation Learning in Offline Setting}
\label{sec:offline}

In this section, we study representation learning in the offline setting. We consider the setting where the offline data does not have a full global coverage. %
We present our algorithm \emph{Lower Confidence Bound driven Representation Learning in offline RL} (\ourofflinealg) in \pref{alg:offline}. 
Our proposed algorithm consists of three parts. The first part is MLE which learns a model $\hat{P}$ and a representation $\hat\phi$. The second part is the construction of a penalty term $\hat b$. Using the learned representation $\hat \phi$, we use a standard bonus in linear bandits as the penalty term as if $\hat{\phi}$ were the true feature. %
The third part is planning with the learned model $\hat P$ and reward $r-\hat b$. %

\begin{algorithm}[t] 
\caption{LCB-driven Representation Learning in offline RL (\ourofflinealg)}\label{alg:offline}
\begin{algorithmic}[1]
  \STATE {\bf  Input:} Regularizer $\lambda$, Parameter $\alpha$, Model classes $\Mcal=\{\mu^{\top}\phi:\mu\in \Psi,\phi\in \Phi\}$, Dataset $\Dcal$.
  \STATE Learn a model $\hat P $ by MLE: $\hat P=\hat \mu^{\top}\hat \phi=\argmax_{P\in \Mcal}\E_{\Dcal}[\ln P(s'\mid s,a)]$.
  \STATE Set the empirical covariance matrix $ \hat \Sigma=\sum_{(s,a)\in \Dcal}\hat \phi(s,a)\hat \phi^{\top}(s,a)+\lambda I$.
 \STATE Set the reward penalty: 
     \begin{equation*}
      \hat b (s,a)=\min \prns{\alpha \sqrt{\hat \phi(s,a)^{\top}\hat \Sigma^{-1}\hat \phi(s,a)},2}. %
     \end{equation*}
   \STATE Solve $\hat \pi =\argmax_{\pi}V^{\pi}_{\hat P,r-\hat b}$. 
\end{algorithmic}
\end{algorithm}

\subsection{Analysis}

We present the PAC guarantee of \ourofflinealg. Before proceeding, we define a relative condition number as a mean to measure the deviation between a comparator policy $\pi$ and the offline data: 
\begin{align*}
    C^{\star}_{\pi}=\sup_{x\in \RR^d}\frac{x^{\top} \EE_{d^{\pi}_{P^{\star}}}[\phi^{\star}(s,a){\phi^{\star}}^{\top}(s,a) ]x}{x^{\top} \EE_{\rho}[\phi^{\star}(s,a){\phi^{\star}}^{\top} (s,a)] x}.  
\end{align*}\footnote{An equivalent definition is $C^\star_\pi := \tr\big( \EE_{d^{\pi}_{P^{\star}}}[\phi^{\star}(s,a){\phi^{\star}}^{\top}(s,a) ] \big(     \EE_{\rho}[\phi^{\star}(s,a){\phi^{\star}}^{\top} (s,a)]      \big)^{-1} \big)$. }
In the special case where the MDP is just a tabular MDP (i.e., $\phi^\star$ is a one-hot encoding vector), this is reduced to a density ratio $C^{\star}_{\infty}=\max_{s,a} d^{\pi}_{P^{\star}}(s,a)/\rho(s,a)$. 
The relative condition number $C^\star_\pi$ is always no larger than the density ratio and could be much smaller for MDPs with large state spaces. Note that we quantify the relative condition number using the unknown true representation $\phi^\star$.
With the above setup, now we are ready to state the main theorem for \ourofflinealg.

\begin{theorem}[PAC Bound for \ourofflinealg] \label{thm:offline_guaratee}
Let $\omega=\max_{a,s}(1/\pi_b(a\mid s))$.  Denote $\hat\pi$ as the output of \ourofflinealg.
There exists a set of parameters such that with probability at least $1-\delta$, for \emph{any} policy $\pi$ (including history-dependent non-Markovian policies),  
\begin{align*} 
V^{\pi}_{P^{\star},r}-V^{\hat \pi}_{P^{\star},r}\leq   c  \sqrt{\frac{d^4\omega^2 C^{\star}_{\pi}\log(|\Mcal|/\delta)}{(1-\gamma)^4 n}}. 
\end{align*}
\end{theorem}
See Theorem \ref{thm:pac_off} in Appendix B for the detailed parameters.
We explain several implications. First of all, this theorem shows that we can uniformly compete with any policy including history-dependent non-Markovian policies \footnote{Given $\pi=\{\pi_i\}_{i=0}^{\infty}$ where $\pi_i$ depends on $s_0,a_0,\dots s_{i}$, $V^{\pi}_{P^{\star},r}$ and $d^{\pi}_{P^{\star}}(s,a)$ are still well-defined. } satisfying the partial coverage $C^\star_\pi < \infty$. Particularly, if the optimal policy $\pi^\star$ is covered by the offline data, i.e., $C^\star_{\pi^\star} < \infty$, then our algorithm is able to compete against it \footnote{We also require $\omega<\infty$, which is a mild assumption since it does not involve $P^{\star}$. Indeed, it is much weaker than the global coverage type assumption $1/\rho(s,a)<\infty,\forall(s,a)$.  }. Note that assuming offline data covers $\pi^\star$ is still a weaker assumption than the global coverage such as $\sup_{\pi}\sup_{(s,a)}d^{\pi}_{P^{\star}}(s,a)/\rho(s,a)$ in prior offline RL works \citep{antos2008learning,ChenJinglin2019ICiB}. 
Second, our coverage condition is measured by a relative condition number defined using the unknown ground truth representation $\phi^\star$ but not depending on other features. Prior works that use relative condition numbers as measures of coverage are restricted to the settings where the ground truth representation $\phi^\star$ is known
\citep{JinYing2020IPPE,ChangJonathanD2021MCSi,zanette2021provable}. %

To sum up, our algorithm is the first oracle efficient algorithm which does not need to know $\phi^\star$, and  requires partial coverage only in terms of $\phi^{\star}$. Note while \citet{uehara2021pessimistic} has a similar guarantee on low-rank MDPs, their algorithm is not oracle-efficient as it is a version space algorithm.

\paragraph{Highlight of the analysis}
Prior offline RL works that use reward penalties \citep{RashidinejadParia2021BORL,zhang2021corruption,ChangJonathanD2021MCSi} all assume the representation $\phi^\star$ is known a priori, which allows them to use linear regression analysis to derive point-wise model uncertainty quantification which in turn serves as a penalty. In contrast, our function class is nonlinear. As in the online setting, what we can show is near pessimism in the initial state distribution $d_0$. %
\begin{lemma}[Almost Pessimism at the Initial State Distribution]\label{lem:pessimism_main}
There exists a set of parameters that with probability $1-\delta$, we have 
\begin{equation*} 
\forall \pi\in \Pi: V^{\pi}_{\hat P,r-b}-V^{\pi}_{P^{\star},r}\leq c_1\sqrt{\frac{\omega\log(|\Mcal|/\delta)(1-\gamma)^{-1}}{n}}. 
\end{equation*}
\end{lemma}
We leave the detailed proof to \pref{lem:pessimism} in the Appendix~\ref{ape:offline}, which is similar to the proof of \pref{lem:optimism_main}. Though the pessimism in the initial state distribution was recently also used in \cite{XieTengyang2021BPfO,zanette2021provable}, %
the derivation is totally different. %

Similar to the online setting, to obtain a result only depending on the relative condition number using the ground truth representation $\phi^\star$ but not the learned feature $\hat \phi$, we again need to translate the penalty defined with $\hat\phi$ to the potential function $\|\phi^\star(s,a)\|_{\Sigma^{-1}_{\rho,\phi^\star}}$. The same techniques that we used for the online setting can be leveraged here to achieve the above goal.

%% file: conclusion.tex
\section{Conclusion}
We study online/offline RL on low-rank MDPs, where the ground truth feature is not known a priori. For online RL, our new algorithm \ouralg{} significantly improves the sample complexity of the piror state-of-the-art algorithm \flambe{} in all parameters while using the same computational oracles.  \ouralg{} has the best sample complexity among existing oracle efficient algorithms for low-rank MDPs by a margin. Comparing to prior representation learning works on low-rank MDPs and block MDPs that rely on a forward step-by-step reward-free exploration framework, our algorithm interleaves representation learning, exploration, and exploitation together, and learns a single stationary policy.  For offline RL, our new algorithm \ourofflinealg{} is the first oracle efficient algorithm for low-rank MDPs that has a PAC guarantee under a partial coverage condition measured by the relative condition number defined with the true feature representation.

%% file: online_proof.tex
\label{ape:online}

~\paragraph{Notation}
We summarize the notations we frequently use. First of all, hereafter, we assume $c_0,c_1,\cdots,$ are some universal constants, and the notation 
$$f(1/(1-\gamma),|\Acal|,\ln(1/\delta),\ln(|\Mcal|),d,n)\lesssim g(1/(1-\gamma),|\Acal|,\ln(1/\delta),\ln(|\Mcal|),d,n)$$ means 
there exists some constant $c_1>0$, such that
$$f(1/(1-\gamma),|\Acal|,\ln(1/\delta),\ln(|\Mcal|),d,n)\leq c_1 g(1/(1-\gamma),|\Acal|,\ln(1/\delta),\ln(|\Mcal|),d,n)$$ 
for any $0\leq \gamma<1,|\Acal|,\ln(1/\delta),\ln(|\Mcal|),d,n$. 
 
We define
\begin{align*}
    \rho_n (s) \coloneqq \frac{1}{n}\sum_{i=0}^{n-1} d^{\pi_i}_{P^{\star}}(s). 
\end{align*}
With slight abuse of notation, we overload the above notation and use $\rho_n$ for $1/n \sum_{i=0}^{n-1} d^{\pi_i}_{P^{\star}}(s,a)$. Next, define $\rho'_n\in [\Scal \to \RR]$ as a marginal distribution of $s'$ for a triple $$(s,a,s')\sim \rho_n(s)U(a)P^{\star}(s'\mid s,a).$$ 

We define three matrices as follows:
\begin{align*}
    \Sigma_{\rho_n\times U(\Acal),\phi} &=n \EE_{s\sim \rho_n,a\sim U(\Acal)}[\phi(s,a)\phi^{\top}(s,a)] +\lambda_n I , \\ 
        \Sigma_{\rho_n,\phi} &=n \EE_{(s,a)\sim \rho_n}[\phi(s,a)\phi^{\top}(s,a)] +\lambda_n I,  \\ 
      \hat \Sigma_{n,\phi} &=n \EE_{(s,a)\sim \Dcal_n}[\phi\phi^{\top}] +\lambda_n I.
\end{align*}
Note that for a fixed $\phi$, $\hat \Sigma_{n, \phi}$ is an unbiased estimate of $\Sigma_{\rho_n \times U(\Acal), \phi}$.
~
\paragraph{Optimism}

First, we prove the optimism at the initial distribution. This is proved by using a simulation lemma inside the learned model which is important since both the bonus and the learned model use $\hat\phi$. In high level, we will show that the expected bonus $\mathbb{E}_{s,a\sim d^{\pi}_{\hat{P}_n}}\hat{b}_n(s,a)$ is in the same order of the expected model error $\mathbb{E}_{s,a\sim d^{\pi}_{\hat{P}_n}} \| \hat{P}_n(\cdot | s,a) - P^\star(\cdot | s,a) \|_1$. Note that the expectation is with respect to $d^{\pi}_{\hat{P}_n}$.

\begin{lemma}[Almost Optimism at the Initial Distribution]\label{lem:optimism}
Consider an episode $n\,(1\leq n\leq N)$ and set  
\begin{align*}
    \alpha_n = O(\sqrt{\prns{|\Acal|+d^2}\gamma\ln(|\Mcal|n/\delta)  }),\quad \lambda_n=O\prns{d\ln (|\Mcal|n/\delta)},\zeta_n=O\prns{\frac{\ln(|\Mcal|n/\delta)}{n}}. 
\end{align*}
With probability $1-\delta$, we have 
\begin{align*}
  \forall n\in [1,\cdots,N],  \forall \pi \in \Pi, V^{\pi}_{\hat P_n,r+\hat b_n}- V^{\pi}_{P^{\star},r}\geq -\sqrt{(1-\gamma)^{-1}|\Acal|\zeta_n}. 
\end{align*}
\end{lemma}

\begin{proof}
In this proof, letting $f_n(s,a)=\|\hat P_n(\cdot \mid s,a)-P^{\star}(\cdot \mid s,a)\|_1$, we condition on the event 
\begin{align*}
  & \forall n,\quad  \EE_{s\sim \rho_n,a\sim U(\Acal)}[f^2_n(s,a)]\leq \zeta_n, \quad \EE_{s\sim \rho'_n,a\sim U(\Acal)}[f^2_n(s,a)]\leq \zeta_n,\\
  & \forall n,\forall \phi, \|\phi(s,a)\|_{\hat \Sigma^{-1}_n,\phi}=\Theta(\|\phi(s,a)\|_{\Sigma^{-1}_{\rho_n\times U(\Acal),\phi}}). 
\end{align*}

From \pref{lem:con} and \pref{lem:mle}, this event happens with probability $1-\delta$. 
Then, for any policy $\pi$, from simulation lemma \ref{lem:performance}, 
\begin{align}
    & (1-\gamma)(V^{\pi}_{\hat P_n,r+\hat b_n}- V^{\pi}_{P^{\star},r}) \nonumber \\
    &=\EE_{(s,a)\sim d^{\pi}_{\hat P_n }}\bracks{\hat b_n(s,a)+\gamma \EE_{s'\sim \hat P_n(s,a)}\bracks{V^{\pi}_{P^{\star},r}(s') } -\gamma \EE_{s'\sim P^{\star}(s,a)}\bracks{V^{\pi}_{P^{\star},r}(s') } } \nonumber \\
    &\gtrsim \EE_{(s,a)\sim d^{\pi}_{\hat P_n }}\bracks{\min \prns{\alpha_n \|\hat \phi_n(s,a)\|_{ \Sigma^{-1}_{\rho_n\times U(\Acal),\hat \phi_n}},2 }+\gamma \EE_{s'\sim \hat P_n(s,a)}\bracks{V^{\pi}_{P^{\star},r}(s') } -\gamma \EE_{s'\sim P^{\star}(s,a)}\bracks{V^{\pi}_{P^{\star},r}(s') } } \label{eq:simulation_lemma_bonus_model_error}
\end{align} 
where in the last step, we replaced the empirical covariance by the population covariance. Note the notation $\lesssim$ is up to universal constants. Here, since $\|V^{\pi}_{P^\star,r}\|_{\infty} \leq 1$ (since we assume trajectory-wise total reward is normalized between $[0,1]$), we have: 
\begin{align*}
   \left|\EE_{(s,a)\sim d^{\pi}_{\hat P_n }}\braces{\EE_{s'\sim \hat P_n(s,a)}\bracks{V^{\pi}_{P^{\star},r}(s') } - \EE_{s'\sim P^{\star}(s,a)}\bracks{V^{\pi}_{P^{\star},r}(s') }}\right|\leq \EE_{(s,a)\sim d^{\pi}_{\hat P_n }}\braces{f_n(s,a)}. 
\end{align*}
The above is further bounded by  \pref{lem:useful2}: 
\begin{align*}
|\EE_{(s,a)\sim d^{\pi}_{\hat P_n }}\braces{f_n(s,a)}| & \leq 
\EE_{(\tilde s,\tilde a)\sim d^{\pi}_{\hat P_n }} \|\hat \phi_n(\tilde s,\tilde a)\|_{\Sigma_{\rho_n\times U(\Acal),\hat \phi_n}^{-1}}\sqrt{\gamma} \sqrt{\braces{n|\Acal|\EE_{s\sim \rho'_n, a\sim U(\Acal)}\bracks{f^2_n(s,a) }}+  4\lambda_n d+4n \zeta_n}
 \\
&+\sqrt{(1-\gamma)|\Acal|\EE_{s \sim \rho_n,a\sim U(\Acal)}\bracks{f^2_n(s,a) } }. 
\end{align*}
Then, 
\begin{align}\label{eq:useful_opti}
\EE_{(s,a)\sim d^{\pi}_{\hat P_n }}\braces{f_n(s,a)}\lesssim \sqrt{\alpha'_n}\EE_{(\tilde s,\tilde a)\sim d^{\pi}_{\hat P_n }} \|\hat \phi_n(\tilde s,\tilde a)\|_{\Sigma_{\rho_n\times U(\Acal),\hat \phi_n}^{-1}}+\sqrt{|\Acal|\zeta_n(1-\gamma)}. 
\end{align}
where 
\begin{align*}
    \alpha'_n= \gamma \{n|\Acal|\zeta_n + \lambda_n d+n\zeta_n\}\lesssim \gamma \prns{|\Acal|+d^2}\ln(|\Mcal|n/\delta) . 
\end{align*}
Note we here use $f_n(s,a)\leq 2,\E_{s\sim \rho_n,a\sim U(\Acal)}[f_n(s,a)^2 ]\leq \zeta_n$ and $\E_{s\sim \rho'_n,a\sim U(\Acal)}[f_n(s,a)^2 ]\leq \zeta_n$. 

Combining all things together, %
\begin{align}
    & \left|\EE_{(s,a)\sim d^{\pi}_{\hat P_n }}\braces{\EE_{s'\sim \hat P_n(s,a)}\bracks{V^{\pi}_{P^{\star},r}(s') } - \EE_{s'\sim P^{\star}(s,a)}\bracks{V^{\pi}_{P^{\star},r}(s') }} \right|\leq 2\EE_{(s,a)\sim d^{\pi}_{\hat P_n }}\braces{f_n(s,a)} \nonumber \\
   &\lesssim \sqrt{\alpha'_n}\EE_{(\tilde s,\tilde a)\sim d^{\pi}_{\hat P_n }} \|\hat \phi_n(\tilde s,\tilde a)\|_{\Sigma_{\rho_n\times U(\Acal),\hat \phi_n}^{-1}}+  \sqrt{(1-\gamma)|\Acal|\zeta_n}  \nonumber \\
   &\leq \alpha_n \EE_{(\tilde s,\tilde a)\sim d^{\pi}_{\hat P_n }} \|\hat \phi_n(\tilde s,\tilde a)\|_{\Sigma_{\rho_n\times U(\Acal),\hat \phi_n}^{-1}}+ \sqrt{(1-\gamma)|\Acal|\zeta_n},\quad \text{where } \alpha_n :=\sqrt{\alpha'_n}.   \label{eq:useful}
\end{align}

Going back to \pref{eq:simulation_lemma_bonus_model_error}, we have 
\begin{align*}
   & (1-\gamma)(V^{\pi}_{\hat P_n,r+\hat b_n}- V^{\pi}_{P^{\star},r})\\
&\gtrsim \EE_{(s,a)\sim d^{\pi}_{\hat P_n }}\bracks{ \min \prns{\alpha_n \|\hat \phi_n (s,a)\|_{\Sigma_{\rho_n\times U(\Acal),\hat \phi_n}^{-1}},2}+\gamma \EE_{s'\sim \hat P_n(s,a)}\bracks{V^{\pi}_{P^{\star},r}(s') } - \gamma \EE_{s'\sim P^{\star}(s,a)}\bracks{V^{\pi}_{P^{\star},r}(s') } }\\
        &\geq \EE_{(s,a)\sim d^{\pi}_{\hat P_n }}\bracks{\min \prns{\alpha_n \|\hat \phi_n (s,a)\|_{\Sigma_{\rho_n\times U(\Acal),\hat \phi_n}^{-1}}, 2}- \min\prns{\alpha_n \|\hat \phi_n (s,a)\|_{\Sigma_{\rho_n\times U(\Acal),\hat \phi_n}^{-1}}+\sqrt{(1-\gamma)|\Acal|\zeta_n},2 } }\\
        &\geq  -\sqrt{(1-\gamma)|\Acal|\zeta_n}.
\end{align*} 

From the second line to the third line, we again use $\|V^{\pi}_{P^{\star},r}\|_{\infty}= O(1)$ and \pref{eq:useful_opti}.  %
This concludes the proof.
\end{proof}

Next, we obtain the upper bound of $\sum_{n=0}^{N}  V^{\pi^{\star}}_{P^{\star},r}-V^{\pi_n}_{P^{\star},r} $. 
Recall $\pi^{\star}$ is the optimal policy. Though this form is the same as a standard regret form, since we are not exactly deploying $\pi_n$ in episode $n$ (recall that we play a uniform action at the end of the episode), we cannot get the regret guarantee. However, it suffices for the PAC guarantee.

\begin{lemma}[Regret]\label{lem:pseudo_regret}
With probability $1-\delta$, we have 
\begin{align*}
    \sum_{n=1}^{N}  V^{\pi^{\star}}_{P^{\star},r}-V^{\pi_n}_{P^{\star},r}   & \lesssim \sqrt{N \ln \prns{1+\frac{N}{d^2\ln( |\Mcal|/\delta) } }\ln(N|\Mcal|/\delta)}\frac{|\Acal|^2 d^{2}}{(1-\gamma)}. 
\end{align*}
\end{lemma}
\begin{proof}

Similar to Lemma \ref{lem:optimism}, letting $f_n(s,a)=\|\hat P_n(\cdot \mid s,a)-P^{\star}(\cdot \mid s,a)\|_1$, we condition on the event 
\begin{align}\label{eq:conditioning}
  \forall n,\quad  \EE_{s\sim \rho_n,a\sim U(\Acal)}[f^2_n(s,a)]\leq \zeta_n,\quad \forall \phi, \|\phi(s,a)\|_{\hat \Sigma^{-1}_n,\phi}=\Theta(\|\phi(s,a)\|_{\Sigma^{-1}_{\rho_n \times \Ucal,\phi}}). 
\end{align}
From \pref{lem:con} and \pref{lem:mle}, this event happens with probability $1-\delta$. 

For any fixed episode $n$ and any policy $\pi$, we have   
\begin{align*}
       & V^{\pi^{\star}}_{P^{\star},r}-V^{\pi_n}_{P^{\star},r}\\
       &\leq V^{\pi^{\star}}_{\hat P_n,r+\hat b_n}-V^{\pi_n}_{P^{\star},r}+\sqrt{|\Acal| \zeta_n(1-\gamma)^{-1}} \tag{\pref{lem:optimism}}\\ 
       &\leq V^{\pi_n}_{\hat P_n,r+\hat b_n}-V^{\pi_n}_{P^{\star},r}+\sqrt{|\Acal| \zeta_n(1-\gamma)^{-1}} \tag{$\pi_n = \argmax_\pi V^\pi_{\hat P_n,r+\hat b_n}$}\\
     & = (1-\gamma)^{-1} \EE_{(s,a)\sim d^{\pi_n}_{P^{\star} }}[ \hat b_n(s,a)+  \gamma \E_{\hat P_n(s'\mid s,a)}[  V^{\pi_n}_{\hat P_n,r+\hat b_n}(s')]-\gamma \E_{P^{\star}(s'\mid s,a)}[  V^{\pi_n}_{\hat P_n,r+\hat b_n}(s')] ]+\sqrt{|\Acal| \zeta_n(1-\gamma)^{-1} }.  
\end{align*}    
We use the 2nd form of simulation \pref{lem:performance} in the last display. 

Then, noting $\|\hat b_n\|_{\infty}\leq 2$, we have $\|V^{\pi_n}_{\hat P_n,r+\hat b_n}\|_{\infty}\leq 2/(1-\gamma)$. Combining this fact with the above expansion, we have  %
\begin{align}
     &(V^{\pi^{\star} }_{P^{\star},r} -V^{\pi_n}_{P^{\star},r}) \nonumber     \\
    &\leq  (1-\gamma)^{-1} \underbrace{\EE_{(s,a)\sim d^{\pi_n}_{P^{\star} }}[ \hat b_n(s,a)] }_{\text{(a)}}+\prns{\frac{2}{(1-\gamma)^2}}\underbrace{ \EE_{(s,a)\sim d^{\pi_n}_{P^{\star} }}[ f_n(s,a)]}_{\text{(b)}}+\sqrt{|\Acal| \zeta_n(1-\gamma)^{-1}} \label{eq:regret_middle}. 
\end{align}

First, we calculate the first term (a) in Inequality~\ref{eq:regret_middle}. Following \pref{lem:useful} and noting the bonus $\hat b_n$ is $O(1)$, we have  %
\begin{align*}
 & \EE_{(s,a)\sim d^{\pi_n}_{P^{\star} }}\bracks{\hat b_n(s,a)}  \\
 &\lesssim  \EE_{(s,a)\sim d^{ \pi_n}_{P^{\star} }}\bracks{\min\prns{\alpha_n\|\hat \phi_n(s,a)\|_{\Sigma^{-1}_{\rho_n\times \Ucal(\Acal),\hat \phi_n}},2}} \tag{ From \pref{eq:conditioning} }\\ 
 & \lesssim \EE_{(\tilde s,\tilde a)\sim d^{ \pi_n}_{P^{\star} }} \|\phi^{\star}(\tilde s,\tilde a)\|_{\Sigma_{\rho_n,\phi^{\star}}^{-1}}\sqrt{{n\gamma|\Acal|\alpha^2_n}\EE_{s\sim \rho_n,a\sim U(\Acal)}\bracks{ \|\hat \phi_n(s,a)\|^2_{\Sigma^{-1}_{\rho_n\times U(\Acal),\hat \phi_n}}} +d\gamma \lambda_n }\\
   &+\sqrt{|\Acal| \alpha^2_n\EE_{s\sim \rho_n,a\sim U(\Acal)}\bracks{ \|\hat \phi_n(s,a)\|^2_{\Sigma^{-1}_{\rho_n\times U(\Acal),\hat \phi_n}} }(1-\gamma)}.
\end{align*}
Note that we use the fact that $B=2$ when applying \pref{lem:useful}. In addition, we have 
\begin{align*}
     n\EE_{s\sim \rho_n,a\sim  U(\Acal)}\bracks{\|\hat \phi_n(s,a)\|^2_{\Sigma^{-1}_{\rho_n\times U(\Acal),\hat \phi_n}} }=n\Tr(\E_{\rho_n\times U(\Acal)}[\hat \phi_n\hat \phi^{\top}_n]\{n\E_{\rho_n\times U(\Acal)}[\hat \phi_n\hat \phi^{\top}_n]+\lambda_n I\}^{-1} )\leq d.
\end{align*}
Then,
\begin{align*}
       \EE_{(s,a) \sim d^{ \pi_n}_{P^{\star}} }\bracks{\hat b_n(s,a)}
     \leq \EE_{(\tilde s,\tilde a)\sim d^{ \pi_n}_{P^{\star} }} \|\phi^{\star}(\tilde s,\tilde a)\|_{\Sigma_{\rho_n,\phi^{\star}}^{-1}}\sqrt{\gamma d|\Acal|\alpha^2_n+\gamma d\lambda_n}+\sqrt{{d|\Acal|\alpha^2_n}(1-\gamma)/n}. 
\end{align*}

Second, we  calculate the term (b) in inequality~\ref{eq:regret_middle}. Following \pref{lem:useful} and noting $f^2_n(s,a)$ is upper-bounded by $4$ (i.e., $B = 4$ in \pref{lem:useful}), we have 
\begin{align*}
   & \EE_{(s,a)\sim d^{ \pi_n}_{P^{\star} }}[ f_n(s,a)]\\
    &\leq \EE_{(\tilde s,\tilde a)\sim d^{ \pi_n}_{P^{\star} }} \|\phi^{\star}(\tilde s,\tilde a)\|_{\Sigma_{\rho_n,\phi^{\star}}^{-1}}\sqrt{\braces{n|\Acal|\gamma\EE_{ s \sim\rho_n,a\sim U(\Acal) }\bracks{f^2_n(s,a) }}+4\gamma\lambda_n d} \\
    &+ \sqrt{|\Acal|\EE_{s \sim\rho_n,a\sim U(\Acal)}\bracks{{f^2_n(s,a)}{(1-\gamma)}}}\\ 
   &\leq \EE_{(\tilde s,\tilde a)\sim d^{\pi_n}_{P^{\star} }} \|\phi^{\star}(\tilde s,\tilde a)\|_{\Sigma_{\rho_n,\phi^{\star}}^{-1}}\sqrt{n|\Acal|\gamma\zeta_n+ 4\gamma\lambda_n d}+\sqrt{|\Acal|\zeta_n(1-\gamma)}\\
   &\leq \EE_{(\tilde s,\tilde a)\sim d^{\pi_n}_{P^{\star} }} \|\phi^{\star}(\tilde s,\tilde a)\|_{\Sigma_{\rho_n,\phi^{\star}}^{-1}}\alpha_n +\sqrt{|\Acal|\zeta_n(1-\gamma)}, 
\end{align*}
where in the second inequality, we use $\EE_{s \sim\rho_n,a\sim U(\Acal)} [f_n^2(s,a)] \leq \zeta_n$, and 
in the last line, recall $\sqrt{\gamma}\sqrt{{{n|\Acal|}\zeta_n}+\lambda_n d+n\zeta_n}\lesssim \alpha_n$.

Then, by combining the above calculation of the term (a) and term (b) in inequality~\ref{eq:regret_middle}, we have:
\begin{align*}
      V^{\pi^{\star}}_{P^{\star},r} -V^{\pi_n}_{P^{\star},r}   & \lesssim   \frac{1}{(1-\gamma)}\prns{\EE_{(\tilde s,\tilde a)\sim d^{ \pi_n}_{P^{\star} }} \|\phi^{\star}(\tilde s,\tilde a)\|_{\Sigma_{\rho_n,\phi^{\star}}^{-1}}\sqrt{d|\Acal|\alpha^2_n+{d\lambda_n}{}}+\sqrt{\frac{d|\Acal|\alpha^2_n(1-\gamma)}{n}} }\\
      &+ \frac{1}{(1-\gamma)^2}\prns{\EE_{(\tilde s,\tilde a)\sim d^{ \pi_n}_{P^{\star} }} \|\phi^{\star}(\tilde s,\tilde a)\|_{\Sigma_{\rho_n,\phi^{\star}}^{-1}}\alpha_n+\sqrt{|\Acal|\zeta_n(1-\gamma)}}. 
\end{align*}
Hereafter, we take the dominating term out. First,  recall $$\alpha_n\lesssim \sqrt{\braces{|\Acal|+d^2}\ln(N|\Mcal|/\delta))}\lesssim \sqrt{|\Acal| d^2\ln(N|\Mcal|/\delta) }.$$ %
Second, we also use 
\begin{align*}
 &\sum_{n=1}^N \EE_{(\tilde s,\tilde a)\sim d^{ \pi_n}_{P^{\star} }} \|\phi^{\star}(\tilde s,\tilde a)\|_{\Sigma_{\rho_n,\phi^{\star}}^{-1}}   \leq \sqrt{N\sum_{n=1}^N \EE_{(\tilde s,\tilde a)\sim d^{ \pi_n}_{P^{\star} }}[\phi^{\star}(\tilde s,\tilde a)^{\top}\Sigma^{-1}_{\rho_n,\phi^{\star}}\phi^{\star}(\tilde s,\tilde a)^{}]} \tag{CS inequality}\\
& \lesssim \sqrt{N\prns{\ln\det(\sum_{n=1}^N\EE_{(\tilde s,\tilde a)\sim d^{ \pi_n}_{P^{\star} }}[\phi^{\star}(\tilde s,\tilde a)  \phi^{\star}(\tilde s,\tilde a)^{\top} ]  )  - \ln\det(\lambda_1 I)  }  }   \tag{\pref{lem:reduction} and $\lambda_1\leq \cdots\leq \lambda_N$}\\ 
&\leq \sqrt{dN \ln \prns{1+\frac{N}{d\lambda_1} }}.   \tag{Potential function bound, \pref{lem:potential} noting $\|\phi^{\star}(s,a)\|_2\leq 1$ for any $(s,a)$.}
\end{align*}
Finally, 
\begin{align*}
\sum_{n=1}^{N}  V^{\pi}_{P^{\star},r}-V^{\pi_n}_{P^{\star},r} &\lesssim  \frac{1}{(1-\gamma)}\prns{ \sqrt{dN \ln \prns{1+\frac{N}{d\lambda_1} }}\sqrt{{d|\Acal|\alpha^2_N}+d\lambda_N}+\sum_{n=1}^{N}\sqrt{\frac{d|\Acal|\alpha^2_n(1-\gamma)}{n}} }\\
      &+ \frac{1}{(1-\gamma)^2}\prns{ \sqrt{dN \ln \prns{1+\frac{N}{d\lambda_1} }}\alpha_N+\sum_{n=1}^{N}\sqrt{|\Acal|\zeta_n(1-\gamma)}} \\ 
      &\lesssim    \frac{1}{(1-\gamma)}\sqrt{dN \ln \prns{1+\frac{N}{d\lambda_1} }}\sqrt{d|\Acal|\alpha^2_N}+  \frac{1}{(1-\gamma)^2} \sqrt{dN \ln \prns{1+\frac{N}{d\lambda_1} }}\alpha_N \tag{Some algebra. We take the dominating term out. }\\ 
&\lesssim \sqrt{dN \ln \prns{1+\frac{N}{d\lambda_1} }}\frac{|\Acal| d^{3/2}\ln^{1/2}(N|\Mcal|/\delta)}{(1-\gamma)^2} .
\end{align*}
This concludes the proof.
\end{proof}

Using \pref{lem:pseudo_regret}, we can immediately obtain the PAC guarantee. 

\begin{theorem}[PAC guarantee of \ouralg{}]
By interacting with the environment for a number of steps at most
\begin{align*}
  N\log(N/\delta),\quad N\coloneqq O\prns{  \frac{d^4|\Acal|^2  \ln(|\Mcal|/\delta) }{(1-\gamma)^5 \epsilon^2}\ln^{2}\prns{1+  \frac{d^4|\Acal|^2  \ln(|\Mcal|/\delta) }{(1-\gamma)^5 \epsilon^2}}   }. 
\end{align*}
with probability $1-\delta$, we can ensure $V^{\pi^{\star}}_{P^{\star},r}-V^{\hat \pi}_{P^{\star},r} \leq \epsilon. $

\end{theorem}
\begin{proof}
From \pref{lem:pseudo_regret} and \pref{lem:messy}, when $N$ is 
\begin{align*}
O\prns{  \frac{d^4|\Acal|^2 \ln(|\Mcal|/\delta) }{(1-\gamma)^4 \epsilon^2}\ln^{2}\prns{1+  \frac{d^4|\Acal|^2  \ln(|\Mcal|/\delta) }{(1-\gamma)^4 \epsilon^2}}    }, 
\end{align*}
with probability $1-\delta$, we can ensure 
\begin{align*}
            \frac{1}{N}  \sum_{n=1}^{N}  V^{\pi^{\star}}_{P^{\star},r}-V^{\pi_n}_{P^{\star},r} \leq \epsilon. 
\end{align*}

With probability $1-\delta$, we need $(1-\gamma)^{-1}\ln(1/\delta)$ interactions with the environment to get one tuple $(s,a,s',a',\tilde s)$ from one roll-in of $\pi$. Thus, the total sample complexity is $O(N(1-\gamma)^{-1}\ln (N/\delta))$.

\end{proof}

Next, we provide an important lemma to ensure the concentration of the bonus term. The version for fixed $\phi$ is proved in \citet[Lemma 39]{zanette2021cautiously}.  Here, we take a union bound over the whole feature $\phi\in \Phi$. Recall 
\begin{align*}
    \rho_n(\cdot)=\frac{1}{n}\sum_{i=0}^{n-1} d^{ \pi_i}_{P^{\star}}(\cdot ). 
\end{align*}

\begin{lemma}[Concentration of the bonus term] \label{lem:con}
Set $\lambda_n=\Theta(d\ln(n|\Phi|/\delta))$ for any $n$. 
Define
\begin{align*}
    \Sigma_{\rho_n,\phi}=n\EE_{s\sim \rho_n,a \sim U(\Acal)}[\phi(s,a)\phi^{\top}(s,a)]+\lambda_n I,\quad \hat \Sigma_{n,\phi}=\sum_{i=0}^{n-1} \phi(s^{(i)},a^{(i)})\phi^{\top}(s^{(i)},a^{(i)})+\lambda_n I. 
\end{align*} 
With probability $1-\delta$, we have
\begin{align*}
  \forall n \in \NN^{+},\forall \phi\in \Phi, c_1 \|\phi(s,a)\|_{\Sigma^{-1}_{\rho_n\times U(\Acal),\phi}}\leq \|\phi(s,a)\|_{\hat \Sigma^{-1}_{n,\phi}} \leq  c_2 \| \phi(s,a)\|_{\Sigma^{-1}_{\rho_n\times U(\Acal),\phi}}. 
\end{align*}
\end{lemma}

For any $g\in \Scal\times \Acal \to \RR$,  The next lemma shows that $\EE_{(s,a)\sim d^{\pi}_{\hat P_n }}\braces{g(s,a)}$ can be upper-bounded using $\EE_{(\tilde s,\tilde a)\sim d^{\pi}_{\hat P_n }} \|\hat \phi_n(\tilde s,\tilde a)\|_{\Sigma_{\rho_n,\hat \phi_n}^{-1}}$ as long as we have the convergence guarantee for $$\EE_{s\sim \rho_n,a\sim U(\Acal)}\bracks{g^2(s,a) }\,\mathrm{and}\,\EE_{s\sim \rho'_n,a\sim U(\Acal)}\bracks{g^2(s,a) }.$$ 

\begin{lemma}[One-step back inequality for the learned model]\label{lem:useful2}
Take any $g\in \Scal\times \Acal \to \RR$ such that $\|g\|_{\infty}\leq B$. We condition on the event where the MLE guarantee \pref{eq:MLE}: 
\begin{align*}
   \E_{  s \sim \rho_n, a \sim U(\Acal)}[f_n(s,a) ]\lesssim \zeta_n, 
\end{align*}
holds. 
Then, for any policy $\pi$,  we have 
\begin{align*}
&|\EE_{(s,a)\sim d^{\pi}_{\hat P_n }}\braces{g(s,a)}| & \\
&\leq  \EE_{(\tilde s,\tilde a)\sim d^{\pi}_{\hat P_n }} \|\hat \phi_n(\tilde s,\tilde a)\|_{\Sigma_{\rho_n\times U(\Acal),\hat \phi_n}^{-1}} \sqrt{\braces{n|\Acal|\EE_{s\sim \rho'_n, a\sim U(\Acal)}\bracks{g^2(s,a) }}+  B^2\lambda_n d+n B^2\zeta_n}
 \\
&+\sqrt{(1-\gamma)|\Acal|\EE_{s \sim \rho_n,a\sim U(\Acal)}\bracks{g^2(s,a) } }. 
\end{align*}

\end{lemma}

Recall $\Sigma_{\rho_n\times U(\Acal),\hat \phi_n}=n\E_{s \sim \rho_n, a\sim U(\Acal)}[\hat \phi_n(s,a)\hat \phi^{\top}_n (s,a)]+\lambda_n I$. 

\begin{proof}

First, we have an equality: 
\begin{align}\label{eq:first}
    \EE_{(s,a)\sim d^{\pi}_{\hat P_n }}  \braces{g(s,a)}=\gamma   \EE_{(\tilde s,\tilde a)\sim d^{\pi}_{\hat P_n },s\sim \hat P_n(\tilde s,\tilde a),a\sim \pi(s)}\braces{g(s,a)}+(1-\gamma)\EE_{s \sim d_0, a\sim \pi(s_0)}\braces{g(s,a)},
\end{align}

The second term in \pref{eq:first}  is upper-bounded by 
\begin{align*}
&(1-\gamma)\sqrt{\max_{(s,a)}\frac{d_0(s)\pi(a\mid s)}{\rho_n(s)u(a)}\EE_{s\sim \rho_n,a\sim U(\Acal)}\bracks{g^2(s,a) }}\\
&\leq (1-\gamma)  \sqrt{ \max_{(s,a)}\frac{d_0(s)\pi(a\mid s)}{(1-\gamma) d_0(s)u(a)}\EE_{s\sim \rho_n,a\sim U(\Acal)}\bracks{g^2(s,a) }}\leq 
\sqrt{(1-\gamma) |\Acal|\EE_{s\sim \rho_n,a\sim U(\Acal)}\bracks{g^2(s,a) }}. 
\end{align*}

Next we consider the first term in  \pref{eq:first}. By CS inequality, we have 
\begin{align*}
    &\EE_{(\tilde s,\tilde a)\sim d^{\pi}_{\hat P_n },s\sim \hat P_n(\tilde s,\tilde a),a\sim \pi(s)}\braces{g(s,a)}=\EE_{(\tilde s,\tilde a)\sim d^{\pi}_{\hat P_n }}\hat \phi_n(\tilde s,\tilde a)^{\top}\int \sum_{a}\hat \mu_n(s)\pi(a\mid s)g(s,a) d(s)\\ 
   &\leq \EE_{(\tilde s,\tilde a)\sim d^{\pi}_{\hat P_n }} \|\hat \phi_n(\tilde s,\tilde a)\|_{\Sigma_{\rho_n\times U(\Acal),\hat \phi_n}^{-1}}\left \|\int \sum_{a}\hat \mu_n(s)\pi(a\mid s)g(s,a) d(s)\right\|_{\Sigma_{\rho_n\times U(\Acal),\hat \phi_n}}.
\end{align*} 
Then, 
\begin{align*}
  & \|\int \sum_a \hat \mu_n(s)\pi(a\mid s)g(s,a) d(s)\|^2_{\Sigma_{\rho_n\times U(\Acal),\hat \phi_n}}\\
&\leq  \braces{\int \sum_a \hat \mu_n(s)\pi(a\mid s)g(s,a) d(s)}^{\top}\braces{n \EE_{s\sim \rho_n,a\sim U(\Acal)}[\hat \phi_n \hat \phi^{\top}_n]+\lambda_n I  }\braces{\int \sum_a \hat \mu_n(s)\pi(a\mid s)g(s,a) d(s)}\\
&\leq  n \EE_{\tilde s\sim \rho_n,\tilde a\sim U(\Acal)}\braces{\bracks{\int \sum_a \hat \mu_n(s)^{\top}\hat \phi_n(\tilde s,\tilde a)\pi(a\mid s)g(s,a) d(s)}^2}+ B^2\lambda_n d \tag{Use the assumption $\|\sum_a \pi(a\mid s)g(s,a)\|_{\infty}\leq B$ and $\int \|\hat \mu_n(s)h(s)\rd(s)\|_2\leq \sqrt{d}$ for any $h:\Scal \to [0,1]$. }\\
&=  n \EE_{\tilde s\sim \rho_n,\tilde a \sim U(\Acal)}\bracks{\braces{\EE_{s\sim \hat P_n(\tilde s,\tilde a), a\sim \pi(s)}\bracks{g(s,a) }}^2}+ B^2\lambda_n d \\
&\leq  n \EE_{s\sim \rho_n,a\sim U(\Acal)}\bracks{\braces{\EE_{s\sim P^{\star}(\tilde s,\tilde a), a\sim \pi(s)}\bracks{g(s,a) }}^2}+ B^2\lambda_n d + n B^2\zeta_n \tag{MLE guarantee}\\
&\leq  n\EE_{\tilde s\sim \rho_n,\tilde a \sim U(\Acal), s\sim P^{\star}(\tilde s,\tilde a), a\sim \pi(s)}\bracks{g^2(s,a) }+ B^2\lambda_n d+ B^2n\zeta_n.  \tag{Jensen} \\
  &\leq n |\Acal| \braces{\EE_{\tilde s\sim \rho_n,\tilde a \sim U(\Acal), s\sim P^{\star}(\tilde s,\tilde a), a\sim U(\Acal)}\bracks{g^2(s,a) }}+ B^2\lambda_nd+ B^2n\zeta_n  \tag{Importance sampling}\\
&\leq n|\Acal|\EE_{s\sim \rho'_n, a\sim U(\Acal)}\bracks{g^2(s,a) }+ B^2\lambda _n d +B^2 n\zeta_n. \tag{Definition of $\rho'_n$}
\end{align*}

Then, the final statement is immediately concluded. 
\end{proof}

Below, we show a similar lemma as \pref{lem:useful2}. The difference is we aim for calculating $\EE_{(s,a)\sim d^{\pi}_{P^{\star} }}\braces{g(s,a)}$ instead of $\EE_{(s,a)\sim d^{\pi}_{\hat P_n}}\braces{g(s,a)}$ . 
For any $g\in \Scal\times \Acal \to \RR$, this lemma shows that $\EE_{(s,a)\sim d^{\pi}_{P^{\star} }}\braces{g(s,a)}$ can be upper-bounded using $\EE_{(\tilde s,\tilde a)\sim d^{\pi}_{P^{\star} }} \|\phi^{\star}(\tilde s,\tilde a)\|_{\Sigma_{\rho_n,\hat \phi^{\star}}^{-1}}$ as long as we have the convergence guarantee for $\EE_{s \sim \rho_n, a \sim U(\Acal)}\bracks{g^2(s,a) }$. Note comparing to \pref{lem:useful2}, this is not a probabilistic statement. Note that $\|\phi^\star(s,a)\|_{\Sigma^{-1}_{\rho_n,\phi^\star}}$ is the usual elliptical potential function under the fixed representation $\phi^\star$.

\begin{lemma}[One-step back inequality for the true model ]\label{lem:useful} 
Take any $g\in \Scal \times \Acal \to \RR$ such that $\|g \|_{\infty}\leq B$. Then, 
{\small 
\begin{align*}
    \EE_{(s,a)\sim d^{\pi}_{P^{\star} }}\braces{g(s,a)}&\leq \EE_{(\tilde s,\tilde a)\sim d^{\pi}_{P^{\star} }} \|\phi^{\star}(\tilde s,\tilde a)\|_{\Sigma_{\rho_n,\phi{\star}}^{-1}}\sqrt{\gamma}\sqrt{ {n|\Acal|}\EE_{s\sim \rho_n,a\sim U(\Acal)}\bracks{g^2(s,a) }+ \lambda_ndB^2}\\
    &+\sqrt{(1-\gamma)|\Acal| \EE_{s\sim \rho_n,a\sim U(\Acal)}\bracks{g^2(s,a) }}. 
\end{align*}
}
\end{lemma}
Recall $\Sigma_{\rho_n,\phi^{\star}}=n\EE_{(s,a)\sim \rho_n}[\phi^{\star}(s,a)\phi^{\star}(s,a)^{\top}]+\lambda_n I$.
\begin{proof}
First, we have 
\begin{align}\label{eq:first_true}
    \EE_{(s,a)\sim d^{\pi}_{P^{\star} }}\braces{g(s,a)}=\gamma   \EE_{(\tilde s,\tilde a)\sim d^{\pi}_{P^{\star} },s\sim P^{\star}(\tilde s,\tilde a),a\sim \pi(s)}\braces{g(s,a)}+(1-\gamma)\EE_{s \sim d_0, a\sim \pi(s_0)}\braces{g(s,a)}. 
\end{align}
The second term in \pref{eq:first_true}  is upper-bounded by 
\begin{align*}
   (1-\gamma) \sqrt{\max_{(s,a)}\frac{d_0(s)\pi(a\mid s)}{\rho_n(s)u(a)}\EE_{s \sim \rho_n,a\sim U(\Acal)}\bracks{g^2(s,a) }} \leq \sqrt{|\Acal| \EE_{s\sim \rho_n, a\sim U(\Acal)}\bracks{g^2(s,a) }(1-\gamma)}. 
\end{align*}

By CS inequality, the first term in \pref{eq:first_true} is further bounded as follows: 
\begin{align*}
    &\EE_{(\tilde s,\tilde a)\sim d^{\pi}_{P^{\star} },s\sim P^{\star}(\tilde s,\tilde a),a\sim \pi(s)}\braces{g(s,a)}=\EE_{(\tilde s,\tilde a)\sim d^{\pi}_{P^{\star} }}\phi^{\star}(\tilde s,\tilde a)^{\top}\int \sum_a\mu^{\star}(s)\pi(a\mid s)g(s,a) d(s)\\ 
   &\leq \EE_{(\tilde s,\tilde a)\sim d^{\pi}_{P^{\star} }} \|\phi^{\star}(\tilde s,\tilde a)\|_{\Sigma_{ \rho_n, \phi^{\star}}^{-1}}\left \|\int \sum_a \mu^{\star}(s)\pi(a\mid s)g(s,a) d(s)\right \|_{\Sigma_{ \rho_n,\phi^{\star}}}.
\end{align*} 

Here, we have 
\begin{align*}
  & \|\int \sum_a \mu^{\star}(s)\pi(a\mid s)g(s,a) d(s)\|^2_{\Sigma_{ \rho_n,\phi^{\star}}}\\
&\leq  \braces{\int  \sum_a\mu^{\star}(s)\pi(a\mid s)g(s,a) d(s)}^{\top}\braces{ n\EE_{(s,a)\sim  \rho_n}[\phi^{\star}(s,a) \{\phi^{\star}(s,a)\}^{\top} ]+\lambda_n I  }\braces{\int \sum_a\mu^{\star}(s)\pi(a\mid s)g(s,a) d(s)}\\
&\leq  n \EE_{(\tilde s,\tilde a)\sim  \rho_n}\braces{\bracks{\int\sum_a \mu^{\star}(s)^{\top}\phi^{\star}(\tilde s,\tilde a)\pi(a\mid s)g(s,a) d(s)}^2}+ \lambda _ndB^2  \\
&\leq   n \braces{\EE_{(\tilde s,\tilde a)\sim  \rho_n, s\sim P^{\star}(\tilde s,\tilde a), a\sim \pi(s)}\bracks{g^2(s,a) }}+ \lambda _ndB^2 .  \tag{Jensen} 
\end{align*}
Therefore,  %
\begin{align*}
  &  n \braces{\EE_{(\tilde s,\tilde a)\sim  \rho_n, s\sim P^{\star}(\tilde s,\tilde a), a\sim \pi(s)}\bracks{g^2(s,a) }}+ \lambda_n d  B  \\
   &\leq n|\Acal| \braces{\EE_{(\tilde s,\tilde a)\sim  \rho_n, s\sim P^{\star}(\tilde s,\tilde a), a\sim U(\Acal)}\bracks{g^2(s,a) }}+ \lambda_n dB^2  \tag{Importance sampling}\\
  &\leq n|\Acal|\braces{\frac{1}{\gamma}\EE_{s\sim \rho_n,a\sim U(\Acal)}\bracks{g^2(s,a) }}+ \lambda_n dB^2 . 
\end{align*} %
In the last line, we use the following inequality: 
\begin{align*}
    &\EE_{s\sim  \rho_n , a\sim U(\Acal)}\bracks{g^2(s,a) }\\
    &=\gamma \EE_{(\tilde s,\tilde a)\sim  \rho_n, s\sim P^{\star}(\tilde s,\tilde a), a\sim U(\Acal)}\bracks{g^2(s,a) } + (1-\gamma) \EE_{s_0\sim d_0, a\sim U(\Acal)}\bracks{g^2(s,a) } \\ 
    &\geq \gamma \EE_{(\tilde s,\tilde a)\sim  \rho_n, s\sim P^{\star}(\tilde s,\tilde a), a\sim U(\Acal)}\bracks{g^2(s,a) }. 
\end{align*}

Finally, we have
\begin{align*}
  \EE_{(s,a)\sim d^{\pi}_{P^{\star} }}\braces{g(s,a)}&\leq \EE_{(\tilde s,\tilde a)\sim d^{\pi}_{P^{\star} }} \|\phi^{\star}(\tilde s,\tilde a)\|_{\Sigma_{\rho_n,\phi{\star}}^{-1}}\sqrt{\gamma}\sqrt{ \braces{n|\Acal|\EE_{s \sim \rho_n, a \sim U(\Acal) }\bracks{g^2(s,a) }}+ \lambda_n d B^2 }\\
    &+\sqrt{ |\Acal|\EE_{s\sim \rho_n,a\sim U(\Acal)}\bracks{g^2(s,a) }(1-\gamma) }. 
\end{align*}
This concludes the proof.
\end{proof}

%% file: offline_proof.tex
\label{ape:offline}

This section provides the detailed proofs for our results in the offline setting.

Below we first prove that $V^{\pi}_{\hat{P}, r- \hat{b}}$ is an almost pessimistic estimator of $V^{\pi}_{P^\star,r}$. 

\begin{lemma}[Almost Pessimism at the Initial Distribution]\label{lem:pessimism}
Let $\omega=\max_{a,s}1/\pi_b(a\mid s)$. Set 
\begin{align*}
    \alpha = {c_1}\sqrt{\prns{\omega+d^2 } \gamma \ln (|\Mcal|/\delta)},\quad \lambda=O(d\ln (|\Mcal|/\delta)),\zeta=O\prns{\frac{\ln(|\Mcal|/\delta)}{n}}. 
\end{align*}
With probability $1-\delta$, for any policy $\pi$, we have 
\begin{align*}
  V^{\pi}_{\hat P,r-\hat b}- V^{\pi}_{P^{\star},r} \leq \sqrt{\frac{\omega (1-\gamma)^{-1}\ln(|\Mcal|/\delta) }{n}}. 
\end{align*}
\end{lemma}

\begin{proof}
We define 
\begin{align*}
    \Sigma_{\rho,\phi}=n \EE_{ (s,a)\sim \rho}[\phi\phi^{\top}] +\lambda I,\quad   \hat \Sigma_{\phi}=n\E_{\Dcal}[\phi\phi^{\top}]+\lambda I.  
\end{align*}
where $\lambda = O(d\ln (|\Mcal|/\delta))$. In this proof, letting $f(s,a)=\|\hat P(\cdot \mid s,a)-P^{\star}(\cdot \mid s,a)\|_1$, we condition on the events: 
\begin{align}\label{eq:offline_key}
\EE_{(s,a)\sim \rho}[f^2(s,a)]\leq \zeta,\quad \forall \phi\in \Phi: \|\phi(s,a)\|_{\hat \Sigma^{-1}_{\phi}}=\Theta(\|\phi(s,a)\|_{\Sigma^{-1}_{\rho,\phi}}). 
\end{align}
where $\zeta=O(\ln (|\Mcal|/\delta)/n)$. From the offline version of \pref{lem:mle} and \pref{lem:con} \footnote{We can remove $\ln n$ since $n$ is fixed in the offline setting. }, this event happens with probability $1-\delta$.

Then, from simulation  lemma (Lemma \ref{lem:performance}), 
\begin{align*}
   & (1-\gamma) ( V^{\pi}_{\hat P,r-\hat b}- V^{\pi}_{P^{\star},r} ) \\
   &= \EE_{(s,a)\sim d^{\pi}_{\hat P }}\bracks{-\hat b(s,a)+\gamma \EE_{s'\sim \hat P(s,a)}\bracks{V^{\pi}_{P^{\star},r}(s') } - \gamma \EE_{s'\sim P^{\star}(s,a)}\bracks{V^{\pi}_{P^{\star},r}(s') } }\\
    &\lesssim  \EE_{(s,a)\sim d^{\pi}_{\hat P }}\bracks{-\min \prns{\alpha \|\hat \phi(s,a)\|_{ \Sigma^{-1}_{\rho,\hat \phi}},2 }+\gamma \EE_{s'\sim \hat P(s,a)}\bracks{V^{\pi}_{P^{\star},r}(s') } - \gamma \EE_{s'\sim P^{\star}(s,a)}\bracks{V^{\pi}_{P^{\star},r}(s') } } \tag{From \pref{eq:offline_key}}. 
\end{align*} 
Here, we have 
\begin{align*}
   \left |\EE_{(s,a)\sim d^{\pi}_{\hat P }}\braces{\EE_{s'\sim \hat P(s,a)}\bracks{V^{\pi}_{P^{\star},r}(s') } - \EE_{s'\sim P^{\star}(s,a)}\bracks{V^{\pi}_{P^{\star},r}(s') }}\right |\leq \EE_{(s,a)\sim d^{\pi}_{\hat P }}\braces{f(s,a)}, 
\end{align*}
noting $\|V^{\pi}_{P^{\star},r}\|_{\infty}\leq 1$. This is further bounded by  \pref{lem:useful_offline2}: 
\begin{align}\label{eq:useful}
\EE_{(s,a)\sim d^{\pi}_{\hat P }}\braces{f(s,a)}\lesssim \sqrt{\alpha'}\EE_{(\tilde s,\tilde a)\sim d^{\pi}_{\hat P }} \|\hat \phi(\tilde s,\tilde a)\|_{\Sigma_{\rho,\hat \phi}^{-1}}+\sqrt{\omega\zeta (1-\gamma)}. 
\end{align}
where 
\begin{align*}
    \alpha'=n\gamma \omega\zeta +\gamma^2 \lambda d+\gamma^2n\zeta\lesssim \prns{\omega+d^2}\gamma\ln(|\Mcal|/\delta) . 
\end{align*}
Here, we use $f(s,a) \leq 2$ in \pref{lem:useful_offline2} and $\E_{(s,a)\sim \rho}[f^2(s,a)]\leq \zeta$. 

Thus, 
\begin{align*}
    & \left|\EE_{(s,a)\sim d^{\pi}_{\hat P }}\braces{\EE_{s'\sim \hat P(s,a)}\bracks{V^{\pi}_{P^{\star},r}(s') } - \EE_{s'\sim P^{\star}(s,a)}\bracks{V^{\pi}_{P^{\star},r}(s') }} \right|\leq \EE_{(s,a)\sim d^{\pi}_{\hat P }}\braces{f(s,a)} \\
   &\leq  \sqrt{\alpha'}\EE_{(\tilde s,\tilde a)\sim d^{\pi}_{\hat P }} \|\hat \phi(\tilde s,\tilde a)\|_{\Sigma_{\rho,\hat \phi}^{-1}}+  \sqrt{\omega\zeta(1-\gamma)}\\
   &= \alpha \EE_{(\tilde s,\tilde a)\sim d^{\pi}_{\hat P }} \|\hat \phi(\tilde s,\tilde a)\|_{\Sigma_{\rho,\hat \phi}^{-1}}+  \sqrt{\omega\zeta(1-\gamma)},\quad \alpha= \sqrt{\alpha'}. 
\end{align*}

Going back to the simulation lemma \ref{lem:performance}, we have 
\begin{align*}
   & (1-\gamma) ( V^{\pi}_{\hat P,r-\hat b}- V^{\pi}_{P^{\star},r} ) \\
    &\lesssim   \EE_{(s,a)\sim d^{\pi}_{\hat P }}\bracks{ -\min\prns{ \alpha \|\hat \phi (s,a)\|_{\Sigma_{\rho,\hat \phi}^{-1}},2}+\EE_{s'\sim \hat P(s,a)}\bracks{V^{\pi}_{P^{\star},r}(s') } - \EE_{s'\sim P^{\star}(s,a)}\bracks{V^{\pi}_{P^{\star},r}(s') } }\\
        &\leq   \EE_{(s,a)\sim d^{\pi}_{\hat P }}\bracks{- \min\prns{ \alpha \|\hat \phi (s,a)\|_{\Sigma_{\rho,\hat \phi}^{-1}}, 2}+\min \prns{ \alpha \|\hat \phi (s,a)\|_{\Sigma_{\rho,\hat \phi}^{-1}}+\sqrt{\omega\zeta(1-\gamma)} ,2}  } \\
        &\leq \sqrt{\omega\zeta(1-\gamma)}.
\end{align*}
This concludes the proof.
\end{proof}

With the above lemma, now we can proceed to prove the main theorem. 
\begin{theorem}[PAC guarantee of \ourofflinealg]\label{thm:pac_off}
Set the parameters as in \pref{lem:pessimism}. 
With probability $1-\delta$, for any comparator policy $\pi$ including history-dependent non-Markovian policies, we have 
\begin{align*}
    & V^{\pi}_{P^{\star},r}-V^{\hat \pi}_{P^{\star},r} \lesssim \frac{\omega  d^2}{(1-\gamma)^{2} }\sqrt{\frac{C^{\star}_{\pi}\ln(|\Mcal|/\delta)}{n}},
\end{align*} where  $C^\star_{\pi}$ is the relative condition number under $\phi^\star$:
\begin{align*}
    C^{\star}_{\pi}\coloneqq \sup_{x\in \mathbb{R^d}}\frac{x^{\top}\EE_{(s,a)\sim d^{\pi}_{P^{\star}}}[\phi^{\star}(s,a)\{\phi^{\star}(s,a)\}^{\top}]x}{x^{\top}\EE_{ (s,a)\sim \rho}[\phi^{\star}(s,a)\{\phi^{\star}(s,a)\}^{\top}]x}. 
\end{align*}

\end{theorem}

\begin{proof}
In this proof, letting $f(s,a)=\|\hat P(\cdot \mid s,a)-P^{\star}(\cdot \mid s,a)\|_1$
we condition on the events: 
\begin{align}\label{eq:condition_offline}
   \EE_{(s,a)\sim \rho}[f^2(s,a)]\leq \zeta,\quad \forall \phi\in \Phi: \|\phi(s,a)\|_{\hat \Sigma^{-1}_{\phi}}=\Theta(\|\phi(s,a)\|_{\Sigma^{-1}_{\rho,\phi}}). 
\end{align}
From \pref{lem:con} and \pref{lem:mle}, this event happens with probability $1-\delta$. %

For any policy $\pi$, we have  
\begin{align*}
       & V^{\pi}_{P^{\star},r}-V^{\hat \pi}_{P^{\star},r} \\
       &\leq        V^{\pi}_{P^{\star},r}-V^{\hat \pi}_{\hat P,r-\hat b}+ \sqrt{\omega\zeta(1-\gamma)^{-1} }\tag{\pref{lem:pessimism}}\\ 
         &\leq        V^{\pi}_{P^{\star},r}-V^{\pi}_{\hat P,r-\hat b}+ \sqrt{\omega\zeta(1-\gamma)^{-1}} \\ 
    &     \lesssim (1-\gamma)^{-1} \underbrace{\EE_{(s,a)\sim d^{\pi}_{P^{\star} }}[ \hat b(s,a)]}_{\text{(a)}}+\prns{\frac{1}{1-\gamma}}^2 \underbrace{ \EE_{(s,a)\sim d^{\pi}_{P^{\star} }}[ f(s,a)]}_{\text{(b)}}+ \sqrt{\omega\zeta(1-\gamma)^{-1} }.  
 \end{align*} 
 Recall $f(s,a)=\|\hat P(\cdot \mid s,a)-P^{\star}(\cdot \mid s,a)\|_1$. 
 
 From the second line to the third line, note though $\hat \pi$ is the argmax over Markoovian polices, $\hat \pi$ is also the argmax over all history-dependent polices.  In the last line, we use a simulation lemma \ref{lem:performance}, which is tailored to a time-inhomogeneous policy.  We here use $\|V^{\pi}_{\hat P,r-\hat b} \|_{\infty}\leq 2/((1-\gamma))$.  noting $\|\hat b \|_{\infty}=O(1) $. 

We further calculate the first term (a). Considering \ref{lem:useful_offline3} and noting $\|\hat b\|_{\infty}\leq 2$, we have 
\begin{align*}
   \EE_{(s,a)\sim d^{\pi}_{P^{\star} }}[ \hat b(s,a)]&\lesssim  \EE_{(\tilde s,\tilde a)\sim d^{\pi}_{P^{\star} }} \|\phi^{\star}(\tilde s,\tilde a)\|_{\Sigma_{\rho,\phi^{\star}}^{-1}}\sqrt{n\omega \braces{\gamma\EE_{ (s,a)\sim \rho}\bracks{\hat b^2(s,a) }}+\gamma \lambda d } \\ 
   &+\sqrt{\omega(1-\gamma)} \{\E_{\rho}[\hat b^2(s,a)]\}^{1/2}.
\end{align*}
From \pref{eq:condition_offline},  we have 
\begin{align} \label{eq:basic_}
  & n \EE_{ (s,a)\sim \rho}\bracks{\hat b^2(s,a) }\leq   n \EE_{ (s,a)\sim \rho}\bracks{\min\prns{\alpha^2 \|\hat \phi(s,a)\|^2_{\Sigma^{-1}_{\rho,\hat \phi}}  ,4} }   \leq  n \EE_{ (s,a)\sim \rho}\bracks{\alpha^2 \|\hat \phi(s,a)\|^2_{\Sigma^{-1}_{\rho,\hat \phi}} } \\
    &\leq \Tr[n\E_{ (s,a)\sim \rho}[\hat \phi\hat \phi^{\top}]\{n\E_{ (s,a)\sim \rho}[\hat \phi\hat \phi^{\top}]+\lambda I \}^{-1} \\
    &\leq  \Tr[n(\E_{ (s,a)\sim \rho}[\hat \phi\hat \phi^{\top}]+\lambda I)\{n\E_{ (s,a)\sim \rho}[\hat \phi\hat \phi^{\top}]+\lambda I \}^{-1}]\leq d.   
\end{align}
Thus, 
\begin{align*}
   \EE_{(s,a)\sim d^{\pi}_{P^{\star} }}[ \hat b(s,a)]&\leq  \EE_{(\tilde s,\tilde a)\sim d^{\pi}_{P^{\star} }} \|\phi^{\star}(\tilde s,\tilde a)\|_{\Sigma_{\rho,\phi^{\star}}^{-1}}\sqrt{\omega d\alpha^2\gamma+ \gamma\lambda d }+\sqrt{\frac{\omega d\alpha^2(1-\gamma)}{n}}. 
\end{align*}

Second, we further calculate the second term (b). Considering the offline version of \pref{lem:useful} and noting $f^2(s,a)$ is upper-bounded by $4$, 
\begin{align*}
   & \EE_{(s,a)\sim d^{\pi}_{P^{\star} }}[ f(s,a)]\\
    &= \EE_{(\tilde s,\tilde a)\sim d^{\pi}_{P^{\star} }} \|\phi^{\star}(\tilde s,\tilde a)\|_{\Sigma_{\rho,\phi^{\star}}^{-1}}\sqrt{n\omega \braces{\gamma\EE_{ (s,a)\sim \rho}\bracks{f^2(s,a) }}+4\gamma\lambda d}+\sqrt{\omega\EE_{ (s,a)\sim \rho}\bracks{f^2(s,a) }(1-\gamma)} \\ 
   &\lesssim \EE_{(\tilde s,\tilde a)\sim d^{\pi}_{P^{\star} }} \|\phi^{\star}(\tilde s,\tilde a)\|_{\Sigma_{\rho,\phi^{\star}}^{-1}}\sqrt{\omega \braces{n\gamma\zeta}+ \gamma\lambda d}+\sqrt{\omega\zeta(1-\gamma)}\\ 
     &\lesssim \EE_{(\tilde s,\tilde a)\sim d^{\pi}_{P^{\star} }} \|\phi^{\star}(\tilde s,\tilde a)\|_{\Sigma_{\rho,\phi^{\star}}^{-1}}\alpha+\sqrt{\omega\zeta (1-\gamma)}. 
\end{align*}
In the final line, recall $\sqrt{\omega \braces{n\gamma \zeta}+\gamma \lambda d+\gamma n\zeta} \leq \alpha.$

Finally, by combining the calculation of the first term (a) and the second term (b),  we have  
\begin{align*}
        V^{\pi}_{P^{\star},r}-V^{\hat \pi}_{P^{\star},r}  & \lesssim \frac{1}{(1-\gamma)}\EE_{(\tilde s,\tilde a)\sim d^{\pi}_{P^{\star} }} \|\phi^{\star}(\tilde s,\tilde a)\|_{\Sigma_{\rho,\phi^{\star}}^{-1}}\sqrt{d\alpha^2\omega\gamma+\gamma \lambda d}+ \sqrt{\frac{\omega \alpha^2  d(1-\gamma)^{-1}}{n}}\\ 
       & +  \frac{\alpha}{(1-\gamma)^2}\EE_{(\tilde s,\tilde a)\sim d^{\pi}_{P^{\star} }} \|\phi^{\star}(\tilde s,\tilde a)\|_{\Sigma_{\rho,\phi^{\star}}^{-1}}  + \sqrt{\frac{\omega \zeta}{(1-\gamma)^3}}
\end{align*}
Now, we use the fact $\EE_{(\tilde s,\tilde a)\sim d^{\pi}_{P^{\star} }} \|\phi^{\star}(\tilde s,\tilde a)\|_{\Sigma_{\rho,\phi^{\star}}^{-1}}$ is upper-bounded as   
\begin{align*}
   \EE_{(\tilde s,\tilde a)\sim d^{\pi}_{P^{\star} }} \|\phi^{\star}(\tilde s,\tilde a)\|_{\Sigma_{\rho,\phi^{\star}}^{-1}} & \leq   \sqrt{\EE_{(\tilde s,\tilde a)\sim d^{\pi}_{P^{\star} }} \|\phi^{\star}(\tilde s,\tilde a)\|^2_{\Sigma_{\rho,\phi^{\star}}^{-1}}}\leq \sqrt{C^{\star}\EE_{(\tilde s,\tilde a)\sim \rho } \|\phi^{\star}(\tilde s,\tilde a)\|^2_{\Sigma_{\rho,\phi^{\star}}^{-1}} } \tag{Refer to \pref{lem:conversion} }\\
   &\leq \sqrt{C^{\star}d/n}.  \tag{From \pref{eq:basic_}}
\end{align*}
Finally, we have 
\begin{align*}
       & V^{\pi}_{P^{\star},r}-V^{\hat \pi}_{P^{\star},r}\\
        & \lesssim   (1-\gamma)^{-1}\left \{ \sqrt{\frac{C^{\star}d}{n}} \sqrt{d\alpha^2\omega \gamma+ \gamma\lambda d } +  \frac{\alpha}{(1-\gamma)}\sqrt{\frac{C^{\star}d}{n}}  + \sqrt{\frac{\omega \alpha^2  d(1-\gamma)}{n}}+\sqrt{\frac{\omega \zeta}{(1-\gamma)}} \right \} \\ 
        & \lesssim  (1-\gamma)^{-1}\left \{ \sqrt{\frac{C^{\star}d}{n}}\sqrt{{d\alpha^2\omega}{\gamma}}+ \frac{\alpha}{(1-\gamma)}\sqrt{\frac{C^{\star}d}{n}} \right \}\tag{Take out two dominating terms}\\
        &\lesssim \frac{\omega  d^2}{(1-\gamma)^2 }\sqrt{\frac{C^{\star}\ln(|\Mcal|/\delta)}{n}}. 
\end{align*}
\end{proof}

The lemma below is a key technical lemma for our proof. It shows that one can relate the expected value of any function $f(s,a)$ with respect to $d^{\pi}_{\hat{P}}$ (i.e., inside the learned model $\hat{P}$) to the potential function with respect to $d^{\pi}_{\hat P}$, i.e., $\EE_{(\tilde s,\tilde a)\sim d^{\pi}_{\hat P }} \|\hat \phi(\tilde s,\tilde a)\|_{\Sigma_{\rho,\hat \phi}^{-1}}$. Pairing $\hat{\phi}$ and $\hat{P}$ is important since $\hat{P}$ is the low-rank transition model defined using $\hat\phi$. As we have seen in the above analysis, when using the lemma below, we instantiate $f(s,a) := \| \hat{P}(\cdot|s,a) - P^\star(\cdot | s,a) \|_1$.

\begin{lemma}[One-step back inequality for the learned model in offline setting]\label{lem:useful_offline2}
Take any $f\subset \Scal\times \Acal \to \RR$ s.t.  $\|f\|_{\infty}\leq B$. We condition on the event where the MLE guarantee holds:
\begin{align*}
    \E_{(s,a)\sim \rho}\|\hat P(\cdot \mid s,a)-P^{\star}(\cdot \mid s,a)\|^2_1\lesssim \zeta. 
\end{align*}
Then, letting $\omega=\max_{s,a}(1/\pi_b(a\mid s))$, for any policy $\pi$,  we have 
\begin{align*}
|\EE_{(s,a)\sim d^{\pi}_{\hat P }}\braces{f(s,a)}| & \leq 
\EE_{(\tilde s,\tilde a)\sim d^{\pi}_{\hat P }} \|\hat \phi(\tilde s,\tilde a)\|_{\Sigma_{\rho,\hat \phi}^{-1}} \sqrt{\braces{n\omega\EE_{ (s,a)\sim \rho}\bracks{f^2(s,a) }}+ \gamma^2 \lambda d B^2+n \gamma^2 \zeta B^2}
 \\
&+  \sqrt{\omega\EE_{ (s,a)\sim \rho}\bracks{f^2(s,a) }(1-\gamma)  } . 
\end{align*}

\end{lemma}
\begin{proof}

First, we have an equality: 
\begin{align}\label{eq:first_offline}
    \EE_{(s,a)\sim d^{\pi}_{\hat P }}\braces{f(s,a)}=\gamma   \EE_{(\tilde s,\tilde a)\sim d^{\pi}_{\hat P },s\sim \hat P(\tilde s,\tilde a),a\sim \pi(s)}\braces{f(s,a)}+(1-\gamma)\EE_{s \sim d_0, a\sim \pi(s_0)}\braces{f(s,a)}. 
\end{align}

The second term in \pref{eq:first_offline}  is upper-bounded by 
\begin{align*}
\EE_{s \sim d_0, a\sim \pi(s_0)}\braces{f(s,a)}\leq \EE_{s \sim d_0, a\sim \pi(s_0)}\braces{f^2(s,a)}\}^{1/2}=  \sqrt{\omega\EE_{ (s,a)\sim \rho}\bracks{f^2(s,a) }/(1-\gamma)  } . 
\end{align*}

Next we consider the first term in \pref{eq:first_offline}. By CS inequality, we have 
\begin{align*}
    &\left |\EE_{(\tilde s,\tilde a)\sim d^{\pi}_{\hat P },s\sim \hat P(\tilde s,\tilde a),a\sim \pi(s)}\braces{f(s,a)}\right|=\left|\EE_{(\tilde s,\tilde a)\sim d^{\pi}_{\hat P }}\hat \phi(\tilde s,\tilde a)^{\top}\int \sum_a\hat \mu(s)\pi(a\mid s)f(s,a) d(s)\right|\\ 
   &\leq \EE_{(\tilde s,\tilde a)\sim d^{\pi}_{\hat P }} \|\hat \phi(\tilde s,\tilde a)\|_{\Sigma_{\rho,\hat \phi}^{-1}}\|\int \sum_a\hat \mu(s)\pi(a\mid s)f(s,a) d(s)\|_{\Sigma_{\rho,\hat \phi}}.
\end{align*} 
Then, 
\begin{align*}
  & \|\int \hat \mu(s)\pi(a\mid s)f(s,a) d(s,a)\|^2_{\Sigma_{\rho,\hat \phi}}\\
&\leq  \braces{\int \sum_a \hat\mu(s)\pi(a\mid s)f(s,a) d(s)}^{\top}\braces{n \EE_{(s,a)\sim \rho}[\hat \phi \hat \phi^{\top}]+\lambda I  }\braces{\int \sum_a\hat \mu(s)\pi(a\mid s)f(s,a) d(s)}\\
&\leq  n \EE_{(\tilde s,\tilde a)\sim \rho}\braces{\bracks{\int \sum_a \hat \mu(s)^{\top}\hat \phi(\tilde s,\tilde a)\pi(a\mid s)f(s,a) d(s)}^2}+ B^2\lambda d \tag{Use the assumption $\|\sum_a f(s,a)\|_{\infty}\leq B$ and $\|\int \hat \mu(s)h(s)\rd(s)\|_2\leq \sqrt{d}$ for $h:\Scal \to [0,1]$. }\\
&=  n \EE_{(\tilde s,\tilde a)\sim \rho}\{\EE_{s\sim \hat P(\tilde s,\tilde a), a\sim \pi(s)}\bracks{f(s,a) }^2 \}+ B^2\lambda d \\ 
&=  n \EE_{(\tilde s,\tilde a)\sim \rho}\{\EE_{s\sim P^{\star}(\tilde s,\tilde a), a\sim \pi(s)}\bracks{f(s,a) }^2 \}+ B^2\lambda d +n B^2 \zeta \tag{MLE guarantee and $\|\E_{a\sim \pi(\cdot)}[f^2(\cdot,a)]\|_{\infty}\leq B^2.$}\\ 
&\leq  n \braces{\EE_{(\tilde s,\tilde a)\sim \rho, s\sim P^{\star}(\tilde s,\tilde a), a\sim \pi(s)}\bracks{f^2(s,a) }}+ B^2\lambda d+n B^2 \zeta.  \tag{Jensen} 
\end{align*}
Finally, the  first term in \pref{eq:first_offline} is upper-bounded by 
\begin{align*}
   & n  \braces{\EE_{(\tilde s,\tilde a)\sim \rho, s\sim P^{\star}(\tilde s,\tilde a), a\sim \pi(s)}\bracks{f^2(s,a) }}+ \lambda d B^2+n B^2 \zeta \\
   &\leq n\omega \braces{\EE_{(\tilde s,\tilde a)\sim \rho, s\sim P^{\star}(\tilde s,\tilde a), a\sim \pi_b(s)}\bracks{f^2(s,a) }}+ \lambda d  B^2+ nB^2  \zeta \tag{Importance sampling}\\
     &\leq n\omega \braces{\frac{1}{\gamma}\EE_{ (s,a)\sim \rho}\bracks{f^2(s,a) }}+ \lambda d B^2+n B^2 \zeta. \tag{Definition of $\rho$}
\end{align*}
In the last line, we use the following equality: 
\begin{align*}
    \EE_{ (s,a)\sim \rho}\bracks{f^2(s,a) }= \gamma \EE_{(\tilde s,\tilde a)\sim \rho, s\sim P^{\star}(\tilde s,\tilde a), a\sim \pi_b(s)}\bracks{f^2(s,a) } + (1-\gamma) \EE_{ s\sim d_0, a \sim \pi_b}\bracks{f^2(s,a) }. 
\end{align*}

Based on the above discussion, the final statement is immediately concluded.

\end{proof}

We can prove the similar inequality for the true model. The proof is omitted since it is quite similar to the one of Lemma \ref{lem:useful_offline2}. 

\begin{lemma}[One-step back inequality for the true model in offline setting]\label{lem:useful_offline3}
Take any $f\subset \Scal\times \Acal \to \RR$ s.t.  $\|f\|_{\infty}\leq B$. Then, letting $\omega=\max_{s,a}(1/\pi_b(a\mid s))$, for any policy $\pi$,  we have 
\begin{align*}
|\EE_{(s,a)\sim d^{\pi}_{P^{\star} }}\braces{f(s,a)}| & \leq 
\EE_{(\tilde s,\tilde a)\sim d^{\pi}_{P^{\star} }} \| \phi^{\star}(\tilde s,\tilde a)\|_{\Sigma_{\rho,\phi^{\star}}^{-1}} \sqrt{\braces{n\omega\EE_{ (s,a)\sim \rho}\bracks{f^2(s,a) }}+ \gamma^2 \lambda d B^2}
 \\
&+  \sqrt{\omega\EE_{ (s,a)\sim \rho}\bracks{f^2(s,a) }(1-\gamma)  } . 
\end{align*}

\end{lemma}

%% file: auxiliary_lemma.tex
\section{Auxiliary lemmas}

First, we present the MLE guarantee. Regarding the proof, refer to  \citet[Theorem 21]{Agarwal2020_flambe}. 
Note $\hat P_n$ and $\bar \pi_n$ are the quantities appearing in the proposed online algorithm. We can also immediately obtain the statement to the offline case. 

\begin{lemma}[MLE guarantee]\label{lem:mle}
For a fixed episode $n$, with probability $1-\delta$, 
\begin{align*}
   \EE_{s \sim \{0.5\rho_n+0.5\rho'_n\},a\sim U(\Acal) }[ \|\hat P_n(\cdot \mid s,a)-P^{\star}(\cdot \mid s,a)\|^2_1 ] \lesssim  \zeta,\quad \zeta \coloneqq \frac{\ln(|\Mcal|/\delta)}{n}. 
\end{align*}
As a straightforward corollary, with probability $1-\delta$,
\begin{align}
 \forall n\in \NN^{+},  \EE_{s \sim \{0.5\rho_n+0.5\rho'_n\},a\sim U(\Acal) }[ \|\hat P_n(\cdot \mid s,a)-P^{\star}(\cdot \mid s,a)\|^2_1 ] \lesssim  0.5\zeta_n,\quad \zeta_n \coloneqq \frac{\ln(|\Mcal|n/\delta)}{n}.  \label{eq:MLE}
\end{align}
\end{lemma}

The following is a standard inequality to prove regret bounds for linear models. Refer to \citet[Lemma G.2.]{agarwal2020pc}

\begin{lemma}\label{lem:reduction}
Consider the following process. For $n=1,\cdots,N$, $M_n=M_{n-1}+G_n$ with $M_0=\lambda_0 I$ and $G_n$ being a positive semidefinite matrix with eigenvalues upper-bounded by $1$. We have that:
\begin{align*}
    2\ln \det (M_N)-2\ln \det(\lambda_0 I)\geq \sum_{n=1}^N\Tr(G_n M^{-1}_{n-1}). 
\end{align*}
\end{lemma}

\begin{lemma}[Potential function lemma]\label{lem:potential}
Suppose $\Tr(G_n)\leq B^2$. 
\begin{align*}
      2\ln \det (M_N)-2\ln \det(\lambda_0 I)\leq d\ln\prns{1+\frac{NB^2}{d\lambda_0}}. 
\end{align*}
\end{lemma}
\begin{proof}
Let $\sigma_1,\cdots,\sigma_d$ be the set of singular values of $M_N$ recalling $M_N$ is a positive semidefinite matrix. Then, by the AM-GM inequality, 
\begin{align*}
    \ln \det (M_N)/\det(\lambda_0 I)=\ln \prod_{i=1}^d (\sigma_i/\lambda_0) \leq \ln d\prns{\frac{1}{d}\sum_{i=1}^d (\sigma_i/\lambda_0))}
\end{align*}
Since we have $\sum_i \sigma_i=\Tr(M_N)\leq d\lambda_0+NB^2$, the statement is concluded. 
\end{proof}

\begin{lemma}[Simulation lemma]\label{lem:performance} Given two MDPs $(P', r+ b)$ and $(P, r)$, for any policy $\pi$, we have:
\begin{align*}
       V^{\pi}_{P',r+b}-V^{\pi}_{P,r}=\frac{1}{1-\gamma} \E_{(s,a)\sim d^\pi_{P'}}[b(s,a)+\gamma \E_{P'(s'\mid s,a)}[Q^{\pi}_{P,r}(s',\pi)]-\gamma \E_{P(s'\mid s,a)}[Q^{\pi}_{P,r}(s',\pi)] ]
\end{align*}
and
\begin{align*}
    V^{\pi}_{P',r+b}- V^{\pi}_{P,r}=\frac{1}{1-\gamma} \E_{(s,a)\sim d^\pi_{P}}[b(s,a)+\gamma \E_{P'(s'\mid s,a)}[Q^{\pi}_{P,r+b}(s',\pi)]-\gamma \E_{P(s'\mid s,a)}[Q^{\pi}_{P',r+b}(s',\pi)] ]. 
\end{align*}
\end{lemma}
\begin{proof}
We use
\begin{align*}
     V^{\pi}_P-f(s_0,\pi)=\frac{1}{1-\gamma}\E_{d^\pi_P}[r(s,a)+\gamma \E_{P(s'\mid s,a)}[f(s',\pi)]-f(s,a)  ] ]
\end{align*}
Then, 
\begin{align*}
    V^{\pi}_{P',r+b}-V^{\pi}_{P,r}&=\frac{1}{1-\gamma}\E_{(s,a)\sim d^\pi_{P'}}[r(s,a)+b(s,a)+\gamma \E_{P'(s'\mid s,a)}[Q^{\pi}_{P,r}(s',\pi)]-Q^{\pi}_{P,r}(s,a)  ] ]\\
     &=\frac{1}{1-\gamma}\E_{(s,a)\sim d^\pi_{P'}}[b(s,a)+\gamma \E_{P'(s'\mid s,a)}[Q^{\pi}_{P,r}(s',\pi)]-\gamma \E_{P(s'\mid s,a)}[Q^{\pi}_{P,r}(s',\pi)] ]. 
\end{align*}
Similarly, 
\begin{align*}
   V^{\pi}_{P,r}- V^{\pi}_{P',r+b}&=\frac{1}{1-\gamma}\E_{(s,a)\sim d^\pi_{P}}[r(s,a)+\gamma \E_{P(s'\mid s,a)}[Q^{\pi}_{P',r+b}(s',\pi)]-Q^{\pi}_{P',r+b}(s,a)  ] ]\\
     &=\frac{1}{1-\gamma}\E_{(s,a)\sim d^\pi_{P}}[-b(s,a)+\gamma \E_{P(s'\mid s,a)}[Q^{\pi}_{P',r+b}(s',\pi)]-\gamma \E_{P'(s'\mid s,a)}[Q^{\pi}_{P,r}(s',\pi)] ]. 
\end{align*}

\end{proof}

The following lemma is used to deal with the distribution shift in the offline setting. For the proof, refer to \citet{ChangJonathanD2021MCSi}. 

\begin{lemma}[Distribution shift lemma]\label{lem:conversion} Consider any policy $\pi$ and state-action distribution $\rho$, and any representation $\phi^\star$, we have:
\begin{align*} 
\EE_{(s,a)\sim d^{\pi}_{P^{\star}}}[\phi^{\star}(s,a)\{\phi^{\star}(s,a)\}^{\top}]\leq C^{\star}\EE_{\rho}[\phi^{\star}(s,a)\{\phi^{\star}(s,a)\}^{\top}],\quad 
C^{\star}\coloneqq \sup_{x\in \mathbb{R^d}}\frac{x^{\top}\EE_{(s,a)\sim d^{\pi}_{P^{\star}}}[\phi^{\star}\{\phi^{\star}\}^{\top}]x}{x^{\top}\EE_{(s,a)\sim \rho}[\phi^{\star}\{\phi^{\star}\}^{\top}]x}. 
\end{align*}
\end{lemma}

This is some auxiliary lemma to convert the finite sample error bound into the sample complexity. 

\begin{lemma}[Conversion of finite sample error bounds into sample complexities]\label{lem:messy}
By taking $$N= 1/\epsilon'^2\times \ln^2(1+1/\epsilon'^2),\epsilon'=\frac{\epsilon}{a_1\ln^{1/2}(e+a_2)\ln^{1/2}(e+a_3)}.$$ 
It satisfies 
\begin{align*}
    a_1\sqrt{1/N} \ln^{1/2}(1+a_2N)\ln^{1/2}(1+a_3 N) <c \epsilon. 
\end{align*}
where $c$ is a constant independent of $a_1,a_2,a_3$. 
\end{lemma}
\begin{proof}
We first have 
 \begin{align*}
     a_1\sqrt{1/N} \ln^{1/2}(1+a_2N)\ln^{1/2}(1+a_3 N) \leq a_1\max(\ln^{1/2}(1+a_2)\ln^{1/2}(1+a_3),1)\sqrt{1/N}\ln(1+N). 
   \end{align*}
Here, we use 
\begin{align*}
    \ln^{1/2}(1+a_2N)\leq \{\ln(1+a_2)+\ln(1+N) \}^{1/2}\leq \sqrt{\max(1,\ln(1+a_2))\ln(1+N)    }. 
\end{align*}
Then, we prove when $N=1/\epsilon^2\times \ln^2(1+1/\epsilon^2)$. 
\begin{align*}
    \sqrt{1/N}\ln(1+N)<\epsilon. 
\end{align*}
This is proved by 
\begin{align*}
      \sqrt{1/N}\ln(1+N)&\leq \epsilon \times \frac{\ln(1+1/\epsilon^2\times \ln^2(1+1/\epsilon^2))  }{\ln^{}(1+1/\epsilon^2)} \\ 
      &\leq \epsilon \times \frac{\ln(1+1/\epsilon^2)+\ln(1+\ln^{2}(1+1/\epsilon^2))  }{\ln^{}(1+1/\epsilon^2)}\\
     &\leq \epsilon+\epsilon \times \frac{\ln(1+\ln^{2}(1+1/\epsilon^2))  }{\ln^{}(1+1/\epsilon^2)}\\ 
     &\leq \epsilon+\epsilon \times \frac{0.5\{1+\ln^{2}(1+1/\epsilon^2))\}^{1/2}-1  }{\ln^{}(1+1/\epsilon^2)}\\ 
      &\lesssim \epsilon. 
\end{align*}
From the third line to the fourth line, we use $\ln(x)\leq 0.5(x^{1/2}-1)$ for $x>0$. 
Then, the final statement is concluded. 
\end{proof}

%% file: tengyang_paper.tex
\section{More comparison to \citet{XieTengyang2021BPfO}}
\label{sec:comparison}
We briefly explain the guarantee when we use Algorithm 1 \citep{XieTengyang2021BPfO}. For a given reward $r$, we first define a new feature class $\Phi^{+}_r$. 
\begin{definition}[Augmented feature]Let $\phi=[\phi_1,\cdots,\phi_d]$. 
  \begin{align*}
      \Phi^{+}_r=\{\phi^{+}_r;\phi \in \Phi\},\quad \phi^{+}_r= [\phi_1,\cdots,\phi_d,r]. 
  \end{align*}
\end{definition}

Then, we set
\begin{align*}
    \Fcal=\{a^{\top}\phi^{+}_r \mid \|a\|_2\leq c\sqrt{d}+1, \phi^{+}_r \in \Phi^{+}_r\}. 
\end{align*}
where $c$ is some suitable constant. Given the hypothesis class $\Fcal$ for the Q-function,  we can run Algorithm 1 in \citep{XieTengyang2021BPfO}. We denote the output policy as $\hat \pi$.

We check two assumptions to ensure the algorithm works. The first assumption is realizability. This is satisfied since for any policy $\pi\in \Pi$ ($\Pi$ is the class of all Markovian polices), we have $Q^{\pi}_{P^{\star},r}\in \Fcal$. The second assumption is completeness. This is also satisfied since $\Tcal^{\pi}_{P^{\star},r} \Fcal \subset \Fcal$ for any policy $\pi\in \Pi$ where $\Tcal^{\pi}_{P^{\star},r}$ is a Bellman-operator s.t.$$\Tcal^{\pi}_{P^{\star},r}:\{\Scal \times \Acal \to \RR\} \ni f\mapsto r(s,a)+\gamma\E_{s'\sim P^{\star}(s,a)}[f(s',\pi)]\in \{\Scal \times \Acal \to \RR\},$$ where we denote $f(s,\pi) = \EE_{a\sim \pi(s)} f(s,a)$. 
Then, by invoking their Corollary 5, we have 
\begin{theorem}[PAC bound based on \citet{XieTengyang2021BPfO}]\label{thm:tengyang}
With probability $1-\delta$, 
\begin{align*}
 \forall \pi\in \Pi: V^{\pi}_{P^{\star},r}-  V^{\hat \pi}_{P^{\star},r}\leq c \frac{\sqrt{C^{\dagger}_{\pi,r}}}{(1-\gamma)^2}\prns{\frac{(d+1)\log(1/\delta)\log|\Acal|}{n}}^{1/5}. 
\end{align*}
where 
\begin{align*}
    C^{\dagger}_{\pi,r} =\sup_{\phi^{+}_r \in \Phi^{+}_r}\sup_{a \in \RR^{d+1}}\frac{a^{\top}\E_{d^{\pi}_{P^{\star}}}[\phi^{+}_r\{\phi^{+}_r\}^{\top}]a }{a^{\top}\E_{\rho}[\phi^{+}_r\{\phi^{+}_r\}^{\top}]a}. 
\end{align*}
\end{theorem}
 We compare the above result with our result in \pref{thm:offline_guaratee}. First, since $C^{\dagger}_r$ includes $r$ and all possible features in $\Phi$, this partial coverage condition is stronger than ours (recall our partial coverage condition is only related to the true representation $\phi^\star$),
and we always have  $C^{\star}_{\pi}\leq C^{\dagger}_{\pi,r}$. Secondly, the dependence on $n$ is much worse. %
Third, it is unclear whether the learned policy can compete against any history-dependent policy. Recall in \pref{thm:offline_guaratee}, we show that our algorithm can compete with any history-dependent policies.